\documentclass[acmsmall,screen]{acmart}

\usepackage[lined,boxed,commentsnumbered,ruled]{algorithm2e} 
\usepackage{algorithmic}

\usepackage{url}

\usepackage{array}
\newcolumntype{P}[1]{>{\centering\arraybackslash}p{#1}}
\newcolumntype{M}[1]{>{\centering\arraybackslash}m{#1}}
\usepackage{multirow}
\usepackage{makecell}
\usepackage{booktabs}
\usepackage{subcaption}
\usepackage{siunitx}
\usepackage{colortbl}

\setcitestyle{nosort}

\usepackage{amsmath}
\usepackage{amsthm}
\DeclareMathAlphabet{\mathbb}{U}{bbold}{m}{n}
\DeclareMathOperator{\argmax}{arg\,max}

\newtheorem{theorem}{Theorem}
\newtheorem{lemma}{Lemma}

\newtheorem{corollary}{Corollary}
\newtheorem{definition}{Definition}

\AtBeginDocument{%
  }

\setcopyright{acmcopyright}
\copyrightyear{2018}
\acmYear{2018}
\acmDOI{XXXXXXX.XXXXXXX}

\acmJournal{JACM}
\acmVolume{37}
\acmNumber{4}
\acmArticle{111}
\acmMonth{8}

\begin{document}

\title{Effective and Imperceptible Adversarial Textual Attack via Multi-objectivization}

\author{Shengcai Liu}
\email{liusc3@sustech.edu.cn}
\affiliation{
	\department{Research Institute of Trustworthy Autonomous Systems}
	\institution{Southern University of Science and Technology}
	\city{Shenzhen}
	\country{China}
	\postcode{518055}
}

\author{Ning Lu}
\email{11610310@mail.sustech.edu.cn}

\author{Wenjing Hong}
\email{hongwj@sustech.edu.cn}
\affiliation{
	\department{Department of Computer Science and Engineering}
	\institution{Southern University of Science and Technology}
	\city{Shenzhen}
	\country{China}
	\postcode{518055}
}

\author{Chao Qian}
\email{qianc@lamda.nju.edu.cn}
\affiliation{
	\department{National Key Laboratory for Novel Software Technology}
	\institution{Nanjing University}
	\city{Nanjing}
	\country{China}
	\postcode{210023}
}

\author{Ke Tang}
\authornote{Corresponding Author.}
\email{tangk3@sustech.edu.cn}
\affiliation{
 \department{Research Institute of Trustworthy Autonomous Systems}
 \institution{Southern University of Science and Technology}
  \city{Shenzhen}
  \country{China}
  \postcode{518055}
}


\begin{abstract}
The field of adversarial textual attack has significantly grown over the last few years, where the commonly considered objective is to craft adversarial examples (AEs) that can successfully fool the target model.
However, the imperceptibility of attacks, which is also essential for practical attackers, is often left out by previous studies.
In consequence, the crafted AEs tend to have obvious structural and semantic differences from the original human-written text, making them easily perceptible.
In this work, we advocate leveraging multi-objectivization to address such issue.
Specifically, we reformulate the problem of crafting AEs as a multi-objective optimization problem, where the attack imperceptibility is considered as an auxiliary objective.
Then, we propose a simple yet effective evolutionary algorithm, dubbed HydraText, to solve this problem.
To the best of our knowledge, HydraText is currently the only approach that can be effectively applied to both score-based and decision-based attack settings.
Exhaustive experiments involving 44237 instances demonstrate that HydraText consistently achieves competitive attack success rates and better attack imperceptibility than the recently proposed attack approaches.
A human evaluation study also shows that the AEs crafted by HydraText are more indistinguishable from human-written text.
Finally, these AEs exhibit good transferability and can bring notable robustness improvement to the target model by adversarial training.
\end{abstract}

\ccsdesc[500]{Computing methodologies~Natural language processing}
\ccsdesc[300]{Security and privacy}
\ccsdesc[500]{Theory of computation~Evolutionary algorithms}
\ccsdesc[500]{Computing methodologies~Adversarial learning}

\keywords{adversarial textual attack, attack imperceptibility, multi-objectivization, multi-objective optimization}

\received{20 February 2007}
\received[revised]{12 March 2009}
\received[accepted]{5 June 2009}

\maketitle

\newpage

\section{Introduction}
Deep Neural Networks (DNNs) have exhibited vulnerability to \textit{adversarial examples} (AEs) \cite{GoodfellowSS15,SuVS19}, which are crafted by maliciously perturbing the original input to fool the target model.
In the field of Natural Language Processing (NLP), there has been much evidence demonstrating that AEs pose serious threats to security-critical applications.
For examples, spam emails modified by ham and word injection can deceive plenty of machine learning-based security mechanisms \cite{KuchipudiNL20};
alteration of raw bytes of Portable Executable (PE) files can result in the evasion of malware detection systems \cite{GrossePMBM17}.
On the other hand, from the perspective of developing robust systems, AEs can help reveal the weakness of DNNs, and can also be used to improve the robustness of DNNs when included into the training data \cite{WallaceFKGS19}.
Therefore, recently there has been rapidly increasing research interest in devising textual attacks for various NLP tasks \cite{ZhangSAL20}.

According to the accessibility to the target model, textual attacks can be categorized as \textit{white-box} and \textit{black-box} attacks.
White-box attacks require full knowledge of the target model to perform gradient computation \cite{EbrahimiRLD18}.
However, in real-world applications (e.g., online web service such as Google Translate), the model internals are often kept secret.
This fact gives rise to the black-box attacks that only require the output of the target model, which can be further classified into \textit{score-based} \cite{AlzantotSEHSC18} and \textit{decision-based} \cite{Maheshwary2020} attacks.
The former require the class probabilities or confidence scores of the target model, while the latter only need the top label predicted by the model.
Since black-box attacks are generally making weaker and more realistic assumptions than white-box attacks, in this work we focus on black-box score-based and decision-based settings.

Due to the discrete nature of text, crafting textual attacks is challenging, especially in the black-box setting.
Unlike images where small perturbations often have limited impact on the overall quality, for text even slight character-level perturbations can easily break the grammaticality \cite{PruthiDL19}.
For this reason, word substitution has become the dominating paradigm in textual attack, since it can craft AEs with good grammaticality. 
The basic idea of word substitution is to fool the target model by replacing words in the original input with some substitutes.
Concretely, word substitution-based approaches \cite{AlzantotSEHSC18,RenDHC19,Maheshwary2020,ZangQYLZLS20,JinJZS20,LiuLCT2021,Zhou2022} generally adopt a two-step strategy to craft an AE based on an original input $\mathbf{x}_{ori}$.
As shown in Figure~\ref{fig:substitute}, at the first step, they construct a set of candidate substitutes for each word in $\mathbf{x}_{ori}$, resulting in a combinatorial search space of AEs.
To maintain grammaticality of the potential AEs, it is common to use synonyms \cite{AlzantotSEHSC18,RenDHC19,Maheshwary2020} or the words with the same sememes \cite{ZangQYLZLS20} as the candidate substitutes, since they have the same Part-of-Speech (POS) as the original words.\footnote{In grammar, a POS is a category of words (or, more generally, of lexical items) that have similar grammatical properties. Commonly listed English POS are noun, verb, adjective, adverb and so on.} 
{\color{black}
Then at the second step, these approaches would search over the space to find an AE that maximizes some objective function $f$ measuring how well the AE can fool the target model:
\begin{equation}
	\label{eq:single_obj_noconstraint}
	\max_{\mathbf{x} \in \mathbf{X}} f(\mathbf{x}),
\end{equation}
where $\mathbf{X}$ and $\mathbf{x} \in \mathbf{X}$ denote the search space of AEs and an AE, respectively.}

\begin{figure}[tbp]
	\centering
	\scalebox{1.0}{
		\includegraphics[width=1.0\columnwidth]{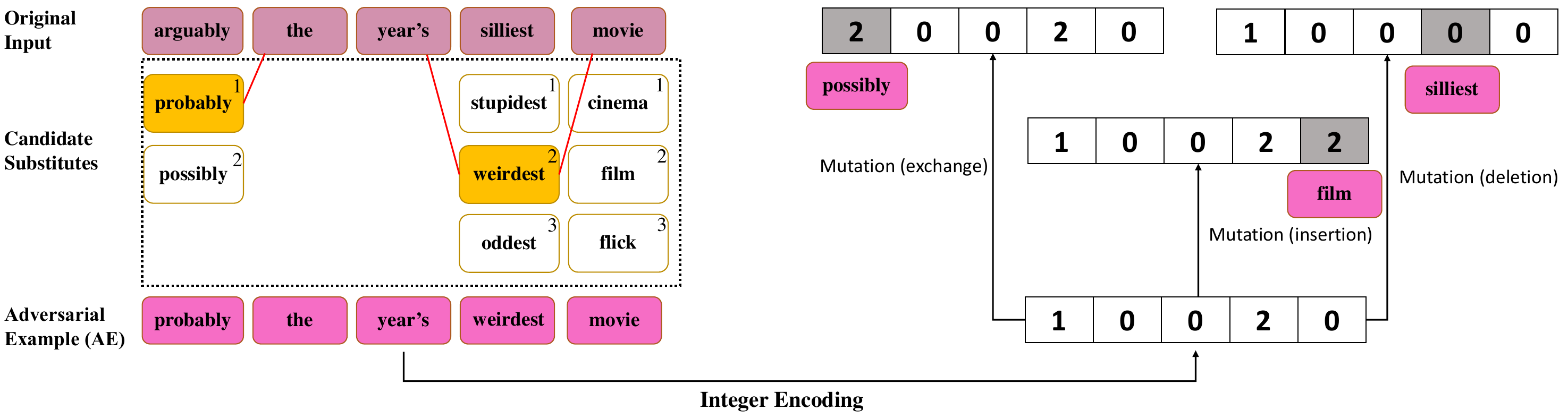}}
	\caption{\color{black}\textbf{On the left side} is an example from the MovieReview (MR) dataset \cite{PangL05}, illustrating the two-step strategy employed by word substitution-based textual attack approaches. The selected substituted words are indicated in yellow; the numbers in the upper right corner of the substituted words indicate their indices.
		\textbf{On the right side} is the solution encoding corresponding to the AE on the left and the solutions generated by the three mutation operators.
		The bits changed by mutation are indicated in gray, and the corresponding substituted words after mutation are also displayed.}
	\label{fig:substitute}
\end{figure}

{\color{black}{Over the past few years, there has been much evidence \cite{RenDHC19,ZangQYLZLS20,LiuLCT2021} showing that word substitution-based approaches can be highly effective in crafting AEs to successfully fool the target model.
However, there still remains a main issue regarding the imperceptibility of attacks.
In brief, attack imperceptibility refers to how likely the crafted AEs are not to be perceived by detectors or humans \cite{DaiL0023}.
Recall that Problem~(\ref{eq:single_obj_noconstraint}) only involves the objective of fooling the target model.
To solve this problem, existing approaches tend to replace an excessive number of words in $\mathbf{x}_{ori}$, or replace the original words with some substitutes that bear significantly different meanings (see Figure~\ref{fig:substitute} and Table~\ref{tab:cases} for some examples).
That is, the crafted AE $\mathbf{x}$ has obvious structural and semantic differences from the original input $\mathbf{x}_{ori}$.}}
In consequence, $\mathbf{x}$ can easily break human prediction consistency (i.e., $\mathbf{x}$ should have the exactly same human predictions as $\mathbf{x}_{ori}$), semantics consistency (i.e., $\mathbf{x}$ should have the same meanings as $\mathbf{x}_{ori}$), and text quality (i.e., naturality).
In real-world applications, such AEs are difficult to pass automatic detectors for anomalous input \cite{Xurong2019,PruthiDL19}, not to mention human judgment.


{\color{black}
A remedy to the above issue is to incorporate a constraint on the similarity between $\mathbf{x}$ and $\mathbf{x}_{ori}$ into the optimization problem:
\begin{equation}
	\label{eq:single_obj_constraint}
	\max_{\mathbf{x} \in \mathbf{X}} f(\mathbf{x})  \ \ \mathrm{s.t.}\ \text{sim}(\mathbf{x},\mathbf{x}_{ori}) > \delta,
\end{equation}
where $\text{sim}(\cdot,\cdot)$ refers to the similarity measure.
One possible way to implement $\text{sim}(\cdot,\cdot)$ is to count the number of common words in $\mathbf{x}$ and $\mathbf{x}_{ori}$.
However, Problem~(\ref{eq:single_obj_constraint}) still has an issue regarding how to set the threshold $\delta$.
It is conceivable that if the value of $\delta$ is too high, it might result in the nonexistence of feasible solutions; whereas if the value of $\delta$ is too low, it might render the similarity constraint practically ineffective.
Further, the appropriate value of $\delta$ actually depends on $\mathbf{x}_{ori}$.
To illustrate, for two $\mathbf{x}_{ori}$s of length 10 and length 100, respectively, the threshold $\delta$ (if implemented using the number of common words) should be set differently.

In this work, we propose to reformulate Problem~(\ref{eq:single_obj_constraint}) into a multi-objective optimization problem (MOP), where the imperceptibility of attacks (i.e., similarity between $\mathbf{x}$ and $\mathbf{x}_{ori}$) is considered as an additional objective.
In the literature, reformulating a single-objective optimization problem (SOP) into a MOP is often referred to as multi-objectivization~\cite{MaHuangLiQiWangZhu2021}.
Compared to Problem~(\ref{eq:single_obj_constraint}), our MOP formulation naturally eliminates the need of (manually) setting the value of $\delta$.
We propose a simple yet effective evolutionary algorithm (EA), dubbed HydraText, to solve the MOP.\footnote{The name is inspired by the Lernaean Hydra, a mythological beast that uses multiple heads to attack its adversaries.}
HydraText can efficiently identify a set of Pareto-optimal AEs, representing various trade-offs between the ability to deceive the target model and the attack imperceptibility.
Furthermore, we demonstrate that determining the final AE from this set is straightforward.}

We conduct a large-scale evaluation of HydraText by using it to attack five modern NLP models across five datasets.
The whole experiments involve \textit{44237} instances in total.
The results show that, compared to five recently proposed textual attack approaches,  HydraText consistently achieves competitive attack success rate and better attack imperceptibility.
A human evaluation study is also conducted to demonstrate that the AEs crafted by HydraText maintain better validity and naturality than the baselines, making the former more indistinguishable from human-written text.
Finally, these AEs can also transfer well to unseen models and can help improve the robustness of the target models by adversarial training.

To the best of our knowledge, HydraText is the \textit{first} approach that utilizes multi-objectivization to adversarial textual attack, and also currently the \textit{only} approach that can effectively craft high-quality AEs in both score-based and decision-based attack settings.
It provides a novel direction to integrate multiple conflicting objectives into the generation of AEs, where more objectives such as fluency can  be further incorporated.
The remainder of the paper is organized as follows.
Section~\ref{sec:related} briefly reviews the related works.
Section~\ref{sec:prob_form} details the formulation of the MOP.
Section~\ref{sec:method} presents the proposed algorithm HydraText.
Section~\ref{sec:exp} presents the experiments, as well as the human evaluation study.
Finally, Section~\ref{sec:conclusion} concludes the paper and discusses potential future directions.

\section{Related Work}
\label{sec:related}
In this section, we briefly review the field of adversarial textual attack, the EAs used in this field, as well as the multi-objectivization in Evolutionary Computation.

\subsection{Adversarial Textual Attack}
As aforementioned, adversarial textual attacks can be broadly classified into white-box and black-box attacks.
White-box attacks require full knowledge of the target model to exploit the gradient of the network with respect to input perturbation.
They are mostly inspired by the fast gradient sign method (FSGM) \cite{GoodfellowSS15} and the forward derivative method (FDM) \cite{PapernotMSH16} which were proposed for attacking image-based systems.
Since gradient computation cannot be directly done on the input text, but only on the embeddings,
some adaptations are necessary for perturbing discrete text \cite{PapernotMSH16,SatoSS018,LSBLS18,WallaceFKGS19,EbrahimiRLD18}.
However, in real-world scenarios where the target model is not fully accessible, white-box attacks are not applicable.
 
Compared to white-box attacks, black-box attacks do not require internals of the target model, but only its output.
According to the amount of information exposed to the attackers from the model output, black-box attacks can be further categorized as score-based and decision-based attacks.
Score-based attacks leverage the confidence scores (or predicted probabilities) of the target model to guide the generation of AEs.
As aforementioned, they are mostly based on word substitution which crafts AEs by replacing words in the original input with some substitutes.
Specifically, after generating candidate substitutes for the words in the original input, these approaches would search over the space of AEs to find one that maximizes the objective function, which assesses how well a given AE can fool the target model.
For example, one commonly used objective function is one minus the predicted probability on the ground truth.
To maximize it, previous studies have used population-based search algorithms such as genetic algorithm (GA) \cite{AlzantotSEHSC18,Zhou2022} and particle swarm optimization (PSO) \cite{ZangQYLZLS20}.
Other approaches mainly use simple heuristics such as importance-based greedy substitution \cite{RenDHC19,JinJZS20} which first sorts words according to their importance and then finds the best substitute for each word in turn.
Note for these approaches the only goal considered during the search for AEs is to successfully fool the target model, while attack imperceptibility is not directly optimized.

To the best of our knowledge, there is only one score-based attack \cite{Mathai2020} that considers simultaneous optimization of multiple objectives during the search for AEs.
However, this work is fundamentally different from ours in both problem formulation and optimization algorithm.
Besides, our approach can also be applied to decision-based setting.

Compared to score-based setting, decision-based setting is more challenging since only the top label predicted by the model is available.
Currently, only a few textual attack approaches are decision-based, e.g., performing perturbations in the continuous latent semantic space \cite{ZhaoDS18} and sentence-level rewriting \cite{SinghGR18}.
Recently, a word substitution-based attack approach was proposed in \cite{Maheshwary2020}, and has achieved the state-of-the-art performance in generating AEs in decision-based setting.
This approach first initializes an AE by randomly replacing words in the original input.
Since this initial AE is very likely to be quite different from the original input, it then uses a GA to maximize the similarity between the former and the latter.
In spirit, this approach is similar to ours because both of them directly consider maximizing the attack imperceptibility  (i.e., maximizing the similarity between AEs and the original input).
However, the approach in \cite{Maheshwary2020} restricts the search space to only AEs that successfully fool the target model.
In contrast, our approach has no such restriction, meaning it involves a larger search space than \cite{Maheshwary2020} and can achieve better attack performance regarding both attack success rate and attack imperceptibility (see the experimental results in Section~\ref{sec:exp}).

\subsection{Evolutionary Algorithms for Adversarial Textual Attack}
Over the past few decades, EAs have achieved great success in solving complicated optimization problems \cite{LiuWTQY15,LorandiCI21,LiuT020,DushatskiyAB21,TangLYY21,LU2023101377,LiuTY21}.
Inspired by natural evolution, the essence of an EA is to represent the solutions as individuals in a population, produce offspring through variations, and select appropriate solutions according to their fitness \cite{KAD2005}.
As powerful search algorithms, EAs have been widely used in crafting textual attacks \cite{AlzantotSEHSC18,ZangQYLZLS20,Maheshwary2020,Zhou2022}, and can generally achieve better performance than simple heuristics \cite{ZangQYLZLS20,Maheshwary2020}, due to their global search capabilities.

{\color{black}EAs are also the popular methods for solving MOPs \cite{ZitzlerTLFF03,knowles2006tutorial,LU2023101377,Li2023,yang2023reducing}, mainly due to two reasons: 
1) they  do not require particular assumptions like differentiability or continuity of the problem;
2) they can find multiple Pareto-optimal solutions in a single algorithm run.}
In textual attack, EAs have been used to solve the MOPs induced in crafting AEs \cite{Mathai2020}.
However, based on the results reported in \cite{Mathai2020}, the approach is still significantly inferior to existing single-objective optimization-based textual attack approaches.
{\color{black} It is worth mentioning that in the vision and speech domain, multi-objective EAs have also been used to generate AEs \cite{SuzukiTO19,Deng2019,KhareAM19,BaiaBP21,BaiaBFMP22}.}

{\color{black}
\subsection{Multi-objectivization}
Multi-objectivization \cite{MaHuangLiQiWangZhu2021} refers to reformulating a SOP into a MOP, which is then solved by a multi-objective optimization algorithm.
Compared to solving the original SOP, solving the reformulated MOP can be advantageous in reducing the number of local optima \cite{KnowlesWC01}, improving the population diversity \cite{lochtefeld2011multiobjectivization}, and leveraging more domain knowledge \cite{SeguraSGL11}.
As a result, the multi-objectivization can usually lead to a better solution to the original SOP.
In this work, we utilize multi-objectivization to involve the constraint on attack imperceptibility as an additional objective.
Compared to the SOP as defined in Problem~(\ref{eq:single_obj_constraint}), the reformulated MOP is advantageous as it naturally eliminates the need of manually setting the value of the threshold $\delta$.
Finally, for a comprehensive survey on multi-objectivization, interested readers may refer to \cite{MaHuangLiQiWangZhu2021}.}


\section{Methods}
\label{sec:prob_form}
Like most previous textual attacks, our approach is based on word substitution.
Formally, let $\mathbf{x}_{ori}=w_1...w_i...w_n$ denote the original input, where $n$ is the input length and $w_i$ is the $i$-th word.
We assume that for each $w_i$, a set of candidate substitutes $B_i=\{s_{1}, s_{2}, ...\}$ have been prepared (note $B_i$ can be empty).
Such candidate sets can be generated by any substitute nomination method such as synonym-based substitution \cite{RenDHC19,JinJZS20} and sememe-based substitution \cite{ZangQYLZLS20}.
{\color{black}
To craft an AE $\mathbf{x}$, for each $w_i$ in $\mathbf{x}_{ori}$, we can select at most one candidate substitute from $B_i$ to replace it (if none is selected, the word remains unchanged).
Hence, $\mathbf{x}$ is encoded as an integer vector of length $n$: $\mathbf{x}=[x_1x_2,...,x_n]$, where $x_i$ represents the index of the selected substitute word for $w_i$ (if none is selected, then $x_i=0$ for brevity).
For example, the AE in Figure~\ref{fig:substitute} is encoded as ``$[10020]$''.
}




\begin{figure}[tbp]
\centering
\scalebox{1.0}{
	\includegraphics[width=.5\columnwidth]{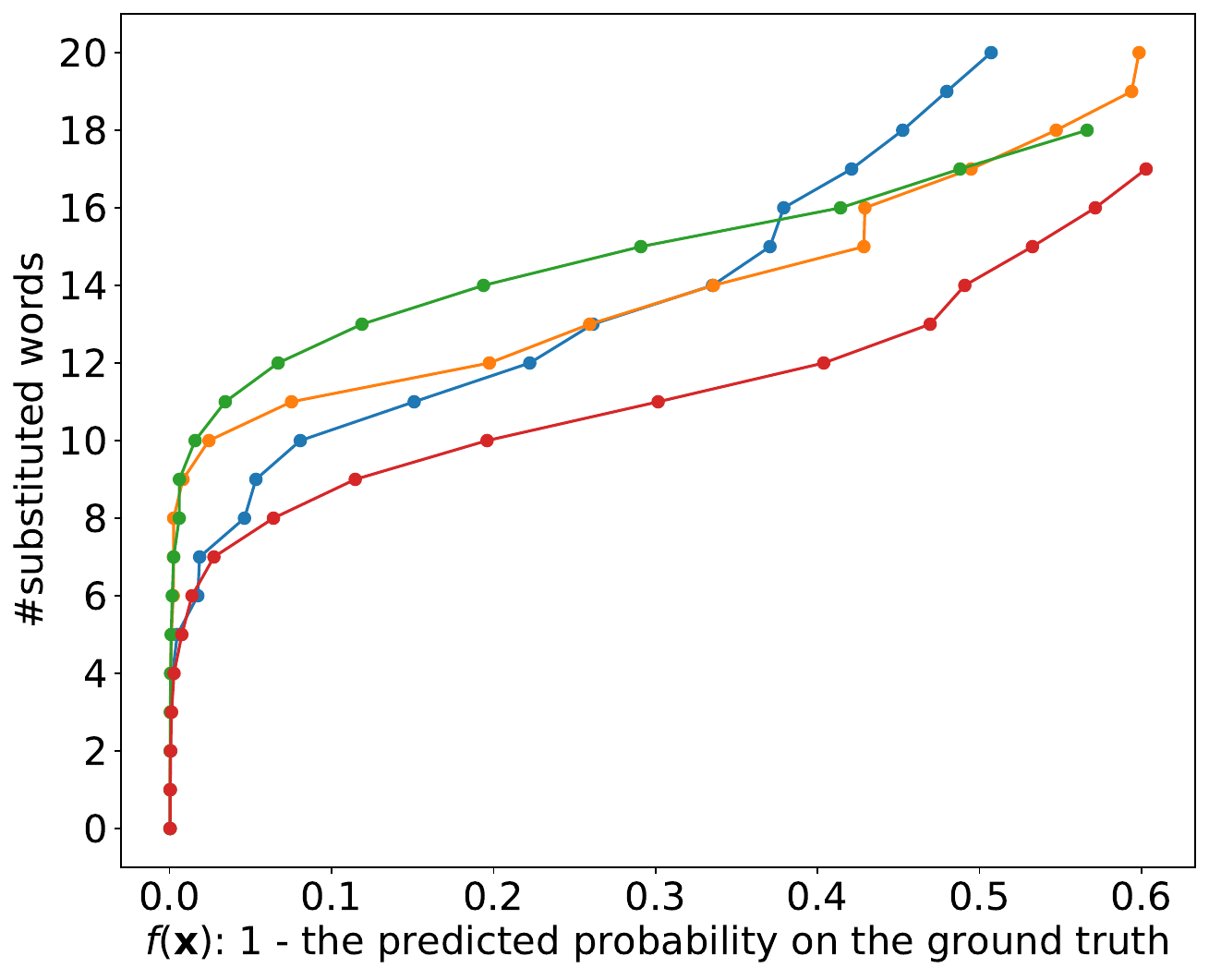}}
\caption{Visualization of $f(\mathbf{x})$ and the number of substituted words (\#substituted words), when applying the greedy algorithm \cite{RenDHC19} to maximize $f(\mathbf{x})$ on four examples  from the IMDB dataset \cite{MaasDPHNP11}, with the target model being BERT. Different examples are indicated by different lines (colors).}
\label{fig:f1_f2}
\end{figure}

{\color{black}
\subsection{Multi-objectivization of Adversarial Textual Attack}}
\label{sec:objective_untarget}
For the sake of brevity, henceforth we will use successful/unsuccessful AEs to denote those AEs that successfully/unsuccessfully fool the target model.
The goal considered by previous score-based attack approaches is to find $\mathbf{x}$ that maximizes an objective function $f(\mathbf{x})$, which is defined as one minus the predicted probability on the ground truth.\footnote{Actually, this definition corresponds to the so-called \textit{untargeted attack}. In the literature, there exists another definition of $f(\mathbf{x})$ corresponding to the so-called \textit{targeted attack}, which is defined as the predicted probability on a particular wrong class. For simplicity, the latter is omitted here, while our approach can also be applied to it (see Section~\ref{sec:targeted}).}
Maximizing only this objective is likely to result in an excessive number of words being substituted and a low semantic similarity between $\mathbf{x}$ and $\mathbf{x}_{ori}$, ultimately making the textual attacks easily perceptible~\cite{abs-2302-02568}.
Figure~\ref{fig:f1_f2} illustrates the changes in $f(\mathbf{x})$ and the number of substituted words,
when applying the greedy algorithm \cite{RenDHC19} to maximize $f(\mathbf{x})$ on four examples from the IMDB dataset \cite{MaasDPHNP11}.
Clearly, as the value of $f(\mathbf{x})$ gets larger, the number of substituted words is also rapidly increasing.

{\color{black}
Motivated by the above observation, we consider the following MOP:
\begin{equation}
\label{eq:problem_definition}
	\max_{\mathbf{x} \in \mathbf{X}} \left(f_1(\mathbf{x}),f_2(\mathbf{x})\right),
\end{equation}
where $f_1(\mathbf{x})$ concerns how well $\mathbf{x}$ fools the target model and $f_2(\mathbf{x})$ concerns attack imperceptibility.
Specifically, in both score-based and decision-based settings, given a solution $\mathbf{x}=[x_1x_2,...,x_n]$, $f_2(\mathbf{x})$ is defined as the 
opposite of the number of substituted words:
\begin{equation}
 	\label{eq:f2}
 	f_2(\mathbf{x})=-\sum_{i=1}^n \mathbb{1}_{x_i\neq0}(x_i),
\end{equation}
where $\mathbb{1}_{x_i\neq0}$ is an indicator function of whether $x_i\neq0$ (i.e., $w_i$ is substituted).

In contrast, the definitions of the first objective $f_1(\mathbf{x})$ are different in score-based setting and decision-based setting.
In score-based setting, it is defined as:
\begin{equation}
	\label{eq:score-based}
	\textbf{Score-based}:\ f_1(\mathbf{x})=\left\{\begin{array}{ll}
		1, & \mathbf{c} \neq \mathbf{c}_{adv}\\
		1-P(\mathbf{x},\mathbf{c}), & \text{otherwise}
	\end{array},\right.
\end{equation}
where $\mathbf{c}$ and $\mathbf{c}_{adv}$ are the ground truth label of $\mathbf{x}_{ori}$ and the label predicted by the target model on $\mathbf{x}$, respectively, and $P(\mathbf{x},\mathbf{c})$ is the probability predicted by the target model that $\mathbf{x}$ belongs to class $\mathbf{c}$.}
Note that the above definition differs from the one considered by previous works (i.e., one minus the predicted probability on the ground truth) in that the objective value is fixed at 1 when $\mathbf{x}$ achieves successful attack.
The intuition is that all successful AEs are equally good at fooling the target model.

In decision-based setting, $P(\mathbf{x},\mathbf{c})$ is unavailable.
One can define $f_1(\mathbf{x})$ as an indicator function, i.e., $\mathbb{1}_{\mathbf{c}_{adv} \neq \mathbf{c}}$.
However, based on this definition, unsuccessful AEs are completely indistinguishable since they all have the same value on $f_1$, making it difficult to identify those ``promising'' ones which are close to the target model's decision boundary.
Moreover, Figure~\ref{fig:f1_f2} clearly shows that the larger the number of substituted words, the closer $\mathbf{x}$ is to successfully fooling the target model.
Therefore, we can actually use the number of substituted words to distinguish those unsuccessful AEs from each other.
{\color{black}
Concretely, in decision-based setting, $f_1(\mathbf{x})$ is defined as follows:
\begin{equation}
\label{eq:decision-based}
  \textbf{Decision-based}:\ f_1(\mathbf{x})=\left\{\begin{array}{ll}
    +\infty, & \mathbf{c} \neq \mathbf{c}_{adv}\\
    \sum_{i=1}^n \mathbb{1}_{x_i\neq0}(x_i), & \text{otherwise}
    \end{array}.\right.
\end{equation}

Problem~(\ref{eq:problem_definition}) is at least NP-hard because even the single-objective version of it is already NP-hard~\cite{LiuLCT2021}.
Solving this MOP needs to compare solutions, which is based on the well-known dominance relation in multi-objective optimization \cite{kaisa1998,Deb2001}, as detailed below.
\begin{definition}[Domination]
	\label{def:domi}
	Given two solutions $\mathbf{x}_1$ and $\mathbf{x}_2$ to Problem~(\ref{eq:problem_definition}), $\mathbf{x}_1$ is said to dominate $\mathbf{x}_2$, denoted as $\mathbf{x}_1 \prec \mathbf{x}_2$, iff for $\forall i \in \{1,2\}, f_i(\mathbf{x}_1) \geq f_i(\mathbf{x}_2)$ and $\exists j \in \{1,2\}, f_j(\mathbf{x}_1) > f_j(\mathbf{x}_2)$.
\end{definition}
Intuitively, the above definition formalizes the notion of ``better'' solutions to Problem~(\ref{eq:problem_definition}).
Based on it, a multi-objective EA would identify a non-dominated solution set, as defined below.

\begin{definition}[Non-dominated Solution Set]
	\label{def:non_domi_set}
	A set $P$ of solutions is a non-dominated solution set iff for any $\mathbf{x} \in P$, $\nexists \mathbf{x}_1 \in P \setminus \{\mathbf{x}\}$ such that $\mathbf{x}_1 \prec \mathbf{x}$.
\end{definition}


One may observe that in decision-based setting, $f_1$ concerns maximizing $\sum_{i=1}^n \mathbb{1}_{x_i\neq0}(x_i)$ (see Eq.~(\ref{eq:decision-based})) while $f_2$  concerns minimizing it (see Eq.~(\ref{eq:f2})).
Given two unsuccessful solutions $\mathbf{x}_1$ and $\mathbf{x}_2$, if they have the same number of substitute words, then we have $f_1(\mathbf{x}_1)=f_1(\mathbf{x}_2) \land f_2(\mathbf{x}_1)=f_2(\mathbf{x}_2)$.
On the other hand, if $\mathbf{x}_1$ has more substituted words than $\mathbf{x}_2$, then we have  $f_1(\mathbf{x}_1)>f_1(\mathbf{x}_2) \land f_2(\mathbf{x}_1)<f_2(\mathbf{x}_2)$.
Hence, in either case, $\mathbf{x}_1$ and $\mathbf{x}_2$ are not comparable.
That is, the objective functions do not exhibit preference between them, which would consequently force the algorithm to keep trying random word substitutions to find successful AEs (see Section~\ref{sec:method}).
Actually, such strategy has been shown effective in finding successful AEs in decision-based setting \cite{Maheshwary2020}.
}

{\color{black}
\subsection{The HydraText Algorithm}
\label{sec:method}

We propose HydraText, a simple yet effective EA, to identify a non-dominated solution set to Problem~(\ref{eq:problem_definition}).
HydraText is mostly inspired by the conventional Global Simple Evolutionary Multiobjective Optimizer (GSEMO) \cite{LaiZHZ14,Qian20}.
Specifically, GSEMO is a simple EA that maintains a population (non-dominated solution set) and in each generation randomly selects a solution from this set and mutates it to generate new solutions. 
Compared to GSEMO, HydraText mainly differs in solution mutations and termination conditions.
Here, it is worth noting that the main goal of this work is to investigate the use of multi-objectivization to adversarial textual attack.
Therefore, a relatively simple multi-objective optimization algorithm (HydraText) is employed.
In the future, we will investigate using more advanced multi-objective EAs, such as NSGA-II~\cite{DebAPM02} and MOEA-D~\cite{ZhangL07}, to further enhance the attack performance.

\begin{algorithm}[tbp]
	\color{black}
	\LinesNumbered
	\KwIn{objective functions $f_1,f_2$; length of the original input $n$}
	\KwOut{$P$, $\mathbf{x}^*$}
	$g \leftarrow 0$, $\mathbf{x}^* \leftarrow \{0\}^n$, $P \leftarrow \{\mathbf{x}^*\}$;\\
	\While{$g<G$}
	{
		Select $\mathbf{x}$ from $P$ uniformly at random;\\
		$\mathbf{x}_1, \mathbf{x}_2, \mathbf{x}_3 \leftarrow $ apply three mutation operators to $\mathbf{x}$;\\
		\tcc{update $P$}
		\For{$\mathbf{x} \in \{\mathbf{x}_1,\mathbf{x}_2,\mathbf{x}_3\}$}
		{
			\If{$\nexists \mathbf{x}' \in P$ {\normalfont such that} $\mathbf{x}'$ {\normalfont is} better {\normalfont than} $\mathbf{x} $}
			{
				$P \leftarrow P \setminus \{\mathbf{x}' | \mathbf{x}' \in P \land (\mathbf{x} \prec \mathbf{x}' \lor (f_1(\mathbf{x})=f_1(\mathbf{x}') \land f_2(\mathbf{x})=f_2(\mathbf{x}')))\}$;\\
				$P \leftarrow P \cup \{\mathbf{x}\}$;\\
			}
		}
		$\mathbf{x}^* \leftarrow \argmax_{\mathbf{x} \in P}f_1(\mathbf{x})$;\\
		\lIf{\normalfont every neighborhood solution of $\mathbf{x}^*$ has been visited}{\textbf{break}}
		$g \leftarrow g+1$;\\
	}
	\Return{$P, \mathbf{x}^*$}
	\caption{HydraText}
	\label{alg:hydratext}
\end{algorithm}
%

As presented in Algorithm~\ref{alg:hydratext}, HydraText starts from the initial solution (an all-zero vector of length $n$) representing the original input $\mathbf{x}_{ori}$ (line 1).
In each generation, HydraText uses mutation operators and dominance-based comparison to improve the solutions in the population $P$ (lines 2-15).
Specifically, a parent solution $\mathbf{x}$ is first randomly selected from $P$ (line 3);
then three offspring solutions $\mathbf{x}_1$, $\mathbf{x}_2$, and $\mathbf{x}_3$ (lines 4) are generated by applying three mutation operators to $\mathbf{x}$;
finally each offspring solution is subject to the following procedure (lines 6-11): if no solution in $P$ is \textit{better} than it (line 7), then it is included into $P$ (line 9), and those solutions in $P$ that are dominated by it or have the same objective value as it on both objective functions are removed (line 8).

\subsubsection{Mutation Operators}
To enable the algorithm to explore a larger solution space, we use three different mutation operators: insertion, deletion, and exchange.
Recall a solution $\mathbf{x}$ is encoded as an integer vector of length $n$, and a zero indicates the word not being substituted.
The mutation operators are as follows.
\begin{enumerate}
	\item \textbf{insertion:} change a randomly selected zero in the vector to a random nonzero integer.
	\item \textbf{deletion:} change a randomly selected nonzero integer in the vector to zero.
	\item \textbf{exchange:} change a randomly selected nonzero integer in the vector to a random nonzero integer.
\end{enumerate}
Figure~\ref{fig:substitute} illustrates the solution encoding and mutation operators.

\subsubsection{Solution Comparison}
HydraText compares solutions mainly based on dominance relation.
Specifically, in line 7 of Algorithm~\ref{alg:hydratext}, a solution $\mathbf{x}'$ is \textit{better} than $\mathbf{x}$ if $\mathbf{x}' \prec \mathbf{x}$.
Additionally, in the case that $f_1(\mathbf{x}_1)=f_1(\mathbf{x}_2) \land f_2(\mathbf{x}_1)=f_2(\mathbf{x}_2)$, to encourage the algorithm to find AEs that are semantically similar to $\mathbf{x}_{ori}$, we adopt the following metric to determine whether $\mathbf{x}'$ is better than $\mathbf{x}$:
\begin{equation}
	\label{eq:f3}
	\text{cos\_sim}(\text{USE}(\mathbf{x}_{ori}),\text{USE}(\mathbf{x}')) > \text{cos\_sim}(\text{USE}(\mathbf{x}_{ori}),\text{USE}(\mathbf{x})).
\end{equation}
Here, USE refers to the embeddings (fixed-size vectors) encoded by the universal sentence encoder \cite{Cer2018}, and \text{cos\_sim} refers to cosine similarity.

\subsubsection{Analysis of the Population $P$ and Determining the Final Solution}

During the algorithm run, the population $P$ is always ensured to be a non-dominated solution set.
Moreover, we have the following useful results that naturally determine the final solution.
\begin{lemma}
	\label{lem:onesol}
	For each possible value $q$ of the objective $f_2$, i.e., $q \in \{0,-1,...,-n\}$, $P$ contains at most one solution $\mathbf{x}$ such that $f_2(\mathbf{x})=q$.
\end{lemma}
\begin{proof}
Recalling that $n$ is the length of the original input, assume $\exists \mathbf{x}_1, \mathbf{x}_2 \in P$ and $\exists q \in \{0,-1,...,-n\}$ such that $f_2(\mathbf{x}_1)=f_2(\mathbf{x}_2)=q$.
It must hold that $f_1(\mathbf{x}_1) \neq f_1(\mathbf{x}_2)$ due to the line 9 of Algorithm~\ref{alg:hydratext}.
Without loss of generality, we assume $f_1(\mathbf{x}_1) > f_1(\mathbf{x}_2)$, then by Definition~\ref{def:domi}, $\mathbf{x}_1$ dominates $\mathbf{x}_2$, which contradicts with the fact that $P$ is a non-dominated solution set.
\end{proof}
Lemma~\ref{lem:onesol} immediately leads to the following corollary on the maximum size of $P$.
\begin{corollary}
	\label{cor:setsize}
	$P$ contains at most $n+1$ solutions.
\end{corollary}
Actually, if we sort the solutions in $P$ in descending order according to $f_2$, then the sorted solution list is exactly in ascending order according to $f_1$, as stated in the following lemma.
\begin{lemma}
	\label{lem:sorted}
	Sort the solutions in $P$ as $\mathbf{x}_{\pi(1)},\mathbf{x}_{\pi(2)},...,\mathbf{x}_{\pi(|P|)}$ such that $f_2(\mathbf{x}_{\pi(1)}) > f_2(\mathbf{x}_{\pi(2)})...>f_2(\mathbf{x}_{\pi(|P|)})$, then it must hold that $f_1(\mathbf{x}_{\pi(1)}) <f_1(\mathbf{x}_{\pi(2)})< ...<f_1(\mathbf{x}_{\pi(|P|)})$.
\end{lemma}
\begin{proof}
	If $\exists \mathbf{x}_{\pi(i)}, \mathbf{x}_{\pi(j)} \in P $ such that $i < j$ and $f_1(\mathbf{x}_{\pi(i)}) \geq f_1(\mathbf{x}_{\pi(j)})$, then by Definition~\ref{def:domi}, $\mathbf{x}_{\pi(j)}$ must be dominated by $\mathbf{x}_{\pi(i)}$ since $f_2(\mathbf{x}_{\pi(i)}) > f_2(\mathbf{x}_{\pi(j)})$, which contradicts with the fact that $P$ is a non-dominated solution set.
\end{proof}
In particular, the solution $\mathbf{x}_{\pi(|P|)}$ is important because it achieves the highest value on $f_1$ among all the solutions in $P$.
In other words, it can fool the target model to the most extent among all the solutions.
For convenience, we denote this solutions as $\mathbf{x}^*$ (line 12 in Algorithm~\ref{alg:hydratext}).
The following theorem indicates that  $\mathbf{x}^*$ is the \textit{only} solution in $P$ that may successfully fool the target model.

\begin{theorem}
	\label{the:returned_sol}
	$P$ contains at most one solution that can successfully fool the target model.
	If $P$ indeed contains one, then the corresponding solution is $\mathbf{x}^*$.
\end{theorem}
\begin{proof}
	By definition of $f_1$ (see Eqs.~(\ref{eq:score-based})-(\ref{eq:decision-based})), successful solutions achieve the same value on $f_1$.
	However, by Lemma~\ref{lem:sorted}, it is impossible that two solutions in $P$ have the same value on $f_1$.
	Hence, $P$ contains at most one successful solution.
	Considering that successful solutions achieve the highest value on $f_1$, if $P$ indeed contains one successful solution, then it must be $\mathbf{x}^*$ because $\mathbf{x}^*$ has the highest value on $f_1$ among all the solutions in $P$.
\end{proof}

It is conceivable that if $P$ contains no successful solution, then returning any of the solutions in $P$ makes no difference.
On the other hand, if $\mathbf{x}^*$ can successfully fool the target model, then it should be returned.
In summary, $\mathbf{x}^*$ could always be returned as the final solution (line 16 in Algorithm~\ref{alg:hydratext}).

The above discussions also imply that the solutions in $P \setminus \{\mathbf{x}^*\}$ actually serve as a rich source for the generation of potentially successful solutions with higher $f_2$ values than $\mathbf{x}^*$.
The reason is as follows.
In the parent selection step (line 3 in Algorithm~\ref{alg:hydratext}), supposing a solution $\mathbf{x} \in P \setminus \{\mathbf{x}^*\}$ is selected to generate offspring solutions $\mathbf{x}_1,\mathbf{x}_2,\mathbf{x}_3$ via the insertion, deletion, and exchange operators, respectively, it then holds that $f_2(\mathbf{x}_1)=f_2(\mathbf{x})-1$, $f_2(\mathbf{x}_2)=f_2(\mathbf{x})+1$, and $f_2(\mathbf{x}_3)=f_2(\mathbf{x})$.
That is, among them only $\mathbf{x}_1$ has lower $f_2$ value than $\mathbf{x}$, and $f_2(\mathbf{x}_1)$ is still not lower than (in most cases, higher than) $f_2(\mathbf{x}^*)$ because $f_2(\mathbf{x}_1)=f_2(\mathbf{x})-1 \geq f_2(\mathbf{x}^*)$.
Considering the probability that the solutions in $P \setminus \{\mathbf{x}^*\}$ are selected as the parent solution is rather high, i.e., $1-1/|P|$, it can be inferred that a vast majority of the offspring solutions generated during the whole algorithm run have higher $f_2$ values than $\mathbf{x}^*$.
Once any of these offspring solutions achieves successful attack, it will dominate $\mathbf{x}^*$ and replace it, indicating HydraText has found a new successful AE with lower number of substituted words than the previous one.
\subsubsection{Termination Conditions and Computational Costs}
Finally, HydraText has two termination conditions.
The first is the maximum number of iterations $G$ (line 2 in Algorithm~\ref{alg:hydratext}), which aims to restrict the overall computational costs of the algorithm.
The second is that all the neighborhood solutions of $\mathbf{x}^*$, i.e., the solutions that can be generated by applying the three types of variations to $\mathbf{x}^*$, have been visited (line 13 in Algorithm~\ref{alg:hydratext}).
When terminated based on the second condition, the solution $\mathbf{x}^*$ is ensured to be a local optimum on $f_1$.
This is important since the maximization of $f_1$ determines the ability of HydraText to craft successful AEs, which is essential for practical adversarial textual attackers.

For adversarial textual attack, model queries, i.e., fitness evaluations (FEs), are expensive and account for the vast majority (>95\%) of the computational costs of the algorithm.
In each generation of HydraText, the algorithm would perform three model queries to evaluate the three offspring solutions, meaning the total query number is strictly upper bounded by $3T$.
In practice, the number of queries performed by HydraText is often much smaller than $3T$ due to the second termination condition (see the experiment results in Section~\ref{sec:exp}).}

\begin{table}[tbp]
	\centering
	\caption{Statistics of the datasets and the test accuracy of the target models.
		``\#Classes'' refers to the number of classes.
		``Avg. \#W'' refers to the average sentence length (number of words).
		``Train'' and ``Test'' denote the number of instances of the training and testing sets, respectively.
		``Model ACC\textbar\#Insts.'' refers to the model's test accuracy and the number of testing instances used for attacking the model, separated by ``\textbar''.}
	\scalebox{0.706}{
		\begin{tabular}{cccccc|ccc}
			\toprule
			Task  & Dataset & Train & Test  & \#Classes & Avg. \#W & BERT ACC\textbar\#Insts.  & WordLSTM ACC\textbar\#Insts. & WordCNN ACC\textbar\#Insts.\\
			\midrule
			\multirow{3}[2]{*}[2pt]{\makecell[c]{Text\\ Classification}} & IMDB  & 25K    & 25K    & 2     & 215   & 90.9\ \textbar\ 2132 & 89.8\ \textbar\ 2067 & 89.7\ \textbar\ 2046 \\
			& MR    & 9K   & 1K   & 2     & 20    & 85.0\ \textbar\ 850 & 80.7\ \textbar\ 766\ & 78.0\ \textbar\ 786\ \\
			& AG News & 120K   & 7.6K  & 4     & 43    & 94.3\ \textbar\ 5000 & 91.3\ \textbar\ 5000 & 91.5\ \textbar\ 5000 \\
			\midrule
			\midrule
			Task  & Dataset & Train & Test  & \#Classes & Avg \#W & BERT  ACC\textbar\#Insts.& Infersent ACC\textbar\#Insts.& ESIM ACC\textbar\#Insts.\\
			\midrule
			\multirow{2}[2]{*}[2pt]{\makecell[c]{Textual\\ Entailment}} & SNLI  & 560K  & 10K   & 3     & 8     & 90.5\ \textbar\ 2433 & 83.4\ \textbar\ 2122 & 86.1\ \textbar\ 2308 \\
			& MNLI  & 433K  & 10K   & 3     & 11    & 84.1\ \textbar\ 5000 & 69.6\ \textbar\ 4024 & 76.6\ \textbar\ 4703 \\
			\bottomrule
	\end{tabular}}
	\label{tab:datas_models}
\end{table}

\section{Experiments}
\label{sec:exp}
The experiments mainly aim to address the following question: whether HydraText could craft more imperceptible textual attacks without sacrificing attack success rates, compared to previous textual attack approaches.
To answer it, we performed a large-scale evaluation of HydraText and five recently proposed textual attack approaches by using them to attack five modern NLP models across five datasets, in both score-based and decision-based attack settings.
Eight different evaluation metrics were adopted to thoroughly assess the attacking ability and attack imperceptibility of these approaches. 
All the codes, datasets, and target models are available at \url{https://anonymous.4open.science/r/HydraText}.


\subsection{Datasets and Target Models}
We considered the following five benchmark datasets.
\begin{enumerate}
	\item AG News \cite{ZhangZL15}: A sentence-level multiclass news classification dataset; each instance is labeled with one of the four topics: World, Sports, Business, and Science/Technology.
	\item IMDB \cite{MaasDPHNP11}: A document-level sentiment classification dataset of movie reviews; each instance is labeled with positive or negative.
	\item MR \cite{PangL05}: A sentence-level sentiment classification dataset of movie reviews; each instance is labeled with positive or negative.
	\item SNLI \cite{BowmanAPM15}: A data set for natural language inference (NLI); each instance consists of a premise-hypothesis sentence pair and the task is to judge whether the second sentence can be derived from entailment, contradiction, or neutral relationship with the first sentence.
	\item MultiNLI \cite{WilliamsNB18}: A multi-genre NLI dataset with coverage of transcribed speech, popular fiction, and government reports; compared to SNLI, it contains more linguistic complexity with various written and spoken English text.
\end{enumerate}

{
\color{black}
Among the five datasets, the first three belong to the task of text classification (including news categorization and sentiment analysis), while the last two belong to the task of textual entailment.
Previous studies~\cite{ZangQYLZLS20,RenDHC19} often selected a limited number (typically 1000) of correctly classified testing instances as the original input for attacking.
However, this can introduce a selection bias for evaluating the attack performance.
To address this issue, we used all the correctly classified testing instances for attacking, which usually led to a significantly larger number of used testing instances than previous studies (as shown in the last column of Table~\ref{tab:datas_models}).

For each dataset, we trained three NLP models using the training instances and then used them as target models in the experiments.
Specifically, for text classification task, the models were word-based convolutional neural network (WordCNN) \cite{Kim14}, word-based long-short term memory (WordLSTM) \cite{hochreiter1997long}, and bidirectional encoder representations from transformers (BERT) \cite{DevlinCLT19};
for textual entailment task, the models were ESIM \cite{ChenZLWJI17}, Infersent \cite{ConneauKSBB17}, and fine-tuned BERT \cite{DevlinCLT19}.
These models employed the three mainstream architectures in NLP: LSTM, Transformer, and CNN, and all of these models could achieve the state-of-the-art testing accuracy.
Thus, the inclusion of these diverse models enabled a comprehensive assessment of the attack approaches against the most popular NLP models.
Table~\ref{tab:datas_models} summarizes all the datasets and target models.
From the perspective of optimization, each unique pair of data set and target model corresponded to a specific sub-class of instances of Problem~(\ref{eq:problem_definition}).
Therefore, the attack approaches would be tested on 15 different sub-classes of problem instances, containing \textit{44237} instances in total, which was expected to be sufficient for assessing their performance~\cite{LiuTL020}.
}

\subsection{Baselines and Algorithm Settings}
We compared HydraText with five recently proposed open-source textual attack approaches.
Specifically, in score-based setting, we considered four approaches, i.e., PSO \cite{ZangQYLZLS20}, GA \cite{AlzantotSEHSC18}, TextFooler \cite{JinJZS20}, and PWWS \cite{RenDHC19}, as baselines.
In decision-based setting where there exists significantly fewer available attack approaches, we chose GADe \cite{Maheshwary2020} as the baseline.
\begin{enumerate}
	\item PSO \cite{ZangQYLZLS20}: A score-based attack approach that uses sememe-based substitution and PSO to search for AEs.
	\item GA \cite{AlzantotSEHSC18}: A score-based attack approach that uses synonym-based substitution and a GA to search for AEs.
	\item TextFooler \cite{JinJZS20}: A score-based attack approach that uses synonym-based substitution and importance-based greedy algorithm to craft AEs.
	\item PWWS \cite{RenDHC19}: A score-based attack approach that uses synonym-based substitution and saliency-based greedy algorithm to craft AEs.
	\item GADe \cite{Maheshwary2020}: A decision-based attack approach that uses synonym-based substitution and a GA to search for AEs.
\end{enumerate}
These approaches have been shown to achieve the state-of-the-art attack performance for various NLP tasks and target models.
More importantly, the optimization algorithms adopted by them represent the two mainstream choices in the literature: population-based algorithms and simple heuristics.
{\color{black}
For all the baselines except PWWS which is parameter-free, the recommended parameter settings from their original publications were used.
Specifically, for PSO \cite{ZangQYLZLS20}, $V_{max}$, $\omega_{max}$, $\omega_{min}$, $P_{max}$, $P_{min}$, and $k$ were set to 1, 0.8, 0.2, 0.8, 0.2 and 2, respectively.
For GA \cite{AlzantotSEHSC18}, $N$, $K$, and $\delta$ were set to 8, 4 and 0.5, respectively.
For TextFooler \cite{JinJZS20}, $N$ and $\delta$  were set to 50 and 0.7, respectively.
For GADe \cite{Maheshwary2020}, $\mathcal{K}$ and $\lambda$ were set to 30 and 25, respectively.

To make a fair comparison, a relatively large budget of model queries (FEs), i.e., 6000, was set to allow all the compared approaches to run to completion to achieve their best possible attack performance.
This means the only parameter of HydraText, i.e., the maximum generation number $G$, was set to 2000.
Note that in the experiments the compared approaches often consumed much fewer model queries than 6000, due to their own termination conditions.}
Since HydraText could be used in combination with any existing substitute nomination method, following PSO and GADe, HydraText used sememe-based substitution in score-based setting and synonym-based substitution in decision-based setting.

\subsection{Evaluation Setup}
Eight different metrics were adopted to assess the performance of attack approaches in terms of attacking ability and attack imperceptibility.
Table~\ref{tab:eval_metrics} summarizes all the metrics.

Attack success rate is the percentage of successful attacks, which is an essential metric for assessing the attacking ability.
Query number is the number of queries consumed by the approaches for attacking an instance.
For attack imperceptibility, we used modification rate and semantic similarity as the basic metrics for automatic evaluation.
The former is the percentage of words in the crafted AEs that differ from the original input, and the latter is the cosine similarity between the embeddings (encoded by USE \cite{Cer2018}) of the original input and the AEs.
Considering that HydraText directly optimizes these two metrics while the baselines do not, we used four advanced metrics, including grammaticality, fluency, naturality, and validity, to further assess how likely the AEs crafted by these approaches are to be perceived by detectors and humans.

\begin{table}[tbp]
	\centering
	\caption{Details of the evaluation metrics. ``Auto'' and ``Human'' refer to automatic and human evaluation, respectively. ``Higher'' and ``Lower'' indicate the
		higher/lower the metric, the better an approach performs.}
	\scalebox{0.8}{
		\begin{tabular}{ccc}
			\toprule
			Metrics  & How to Evaluate? & Better? \\
			\midrule
			Success Rate & Auto & Higher\\
			Query Number & Auto & Lower\\
			Modification Rate & Auto & Lower\\
			Semantic Similarity & Auto & Higher\\
			Grammaticality & Auto (Error Increase Rate) & Lower\\
			Fluency & Auto (Language Model Perplexity) & Lower\\
			Naturality & Human (Naturality Score) & Higher\\
			Validity & Human (Percentage of Valid AEs) & Higher\\
			\bottomrule
	\end{tabular}}
	\label{tab:eval_metrics}
\end{table}

Concretely, grammaticality is measured by the increase rate of grammatical error numbers of AEs compared to the original input, where the number of grammatical errors in a sentence is obtained by LanguageTool.\footnote{https://languagetool.org/}
Fluency is measured by the language model perplexity, where the model is GPT-2 \cite{radford2019language}.
Finally, naturality and validity are evaluated by humans.
The former refers to the naturality score annotated by humans to assess how indistinguishable the AEs are from human-written text, and the latter refers to the percentage of AEs with human prediction consistency (see Section~\ref{sec:human_eval} for the details of human evaluation).

Following previous studies \cite{AlzantotSEHSC18,ZangQYLZLS20,RenDHC19}, we restricted the length of the original input to 10-100 to prevent the approaches from running prohibitively long, and treated the crafted AEs with modification rates higher than 25\% as failed attacks.
Previous studies usually sampled a number (typically 1000) of correctly classified instances from the testing sets as the original input for attacking, which however may introduce \textit{selection bias} for assessing the attack performance.
To avoid this issue, we used \textit{all} the correctly classified testing instances for attacking, which usually leads to much more used instances compared to previous studies (see Table~\ref{tab:datas_models}).
As a result, each approach was tested on \textit{44237} instances in total, which is expected to be sufficient for thoroughly assessing their performance.
Finally, the performance of an approach on different testing instances were aggregated and reported.

\begin{table*}[tbp]
	\centering
	\caption{The attack performance achieved by the compared approaches in terms of attack success rate (Suc.), average modification rate (Mod.), average semantic similarity (Sim.), and average query number (\#Que.), in score-based and decision-based settings.
		{\color{black}On each dataset, the test result in terms of attack success rate, average modification rate, and average semantic similarity is indicated in gray if it was not significantly different from the best result (according to a Wilcoxon signed-rank test with significance level $p= 0.05$).}
		Note for success rate and semantic similarity, the higher the better; while for modification rate and query number, the lower the better.}
	\scalebox{0.68}{
		\begin{tabular}{ccS[table-format=2.2]S[table-format=2.2]S[table-format=1.3]S[table-format=4.1]S[table-format=2.2]S[table-format=2.2]S[table-format=1.3]S[table-format=4.1]S[table-format=2.2]S[table-format=2.2]S[table-format=1.3]S[table-format=4.1]}
			\toprule
			\multicolumn{14}{c}{\textbf{Score-based Setting}} \\
			\midrule
			\multirow{2}[4]{*}{Dataset} & \multirow{2}[4]{*}{Attack} & \multicolumn{4}{c}{BERT} & \multicolumn{4}{c}{WordLSTM} & \multicolumn{4}{c}{WordCNN} \\
			\cmidrule(lr){3-6} \cmidrule(lr){7-10} \cmidrule(lr){11-14}   &       & {Suc. (\%)} & {Mod. (\%)} & {Sim.} &{\#Que.} & {Suc. (\%)} & {Mod. (\%)} & {Sim.} &{\#Que.}  & {Suc. (\%)} & {Mod. (\%)} & {Sim.} &{\#Que.} \\
			\midrule
			\multirow{5}[2]{*}{IMDB} & GA    & 77.72 & 10.44 & 0.719 &3031.6 & 82.44 & 10.14 & 0.714 & 2440.0 & 84.02 & 10.07 & 0.733 & 2351.3 \\
			& PWWS  & 62.38 & 6.83 & 0.726 & {\cellcolor{gray!60}} 122.5 & 70.54 & 6.54 & 0.732 & {\cellcolor{gray!60}} 120.8 & 74.58 & 6.28 & 0.748 & {\cellcolor{gray!60}} 120.1\\
			& TextFooler & 99.25 & 5.81 & 0.787 & 220.9 & {\cellcolor{gray!60}} 99.66 & 5.60 & 0.779 & 214.2 & {\cellcolor{gray!60}} 99.56 & 5.44 & 0.795 & 207.4 \\
			& PSO   & {\cellcolor{gray!60}} 99.91 & 4.54 & 0.771 & 2911.6 & {\cellcolor{gray!60}} 99.85 & 4.50 & 0.768 & 2928.9 & {\cellcolor{gray!60}} 99.90 & 4.37 & 0.785 & 2790.2\\
			& HydraText & {\cellcolor{gray!60}} 100.00 & {\cellcolor{gray!60}} 3.74 & {\cellcolor{gray!60}} 0.851 & 1324.3 & {\cellcolor{gray!60}} 99.90 & {\cellcolor{gray!60}} 3.66 & {\cellcolor{gray!60}} 0.841 & 1175.7 & {\cellcolor{gray!60}} 99.90 & {\cellcolor{gray!60}} 3.54 & {\cellcolor{gray!60}} 0.860 & 1105.0\\
			\midrule
			\multirow{5}[2]{*}{MR} & GA    & 57.29 & 11.29 & 0.575 & 794.9 & 64.71 & 11.68 & 0.668 & 476.3 & 62.14 & 11.16 & 0.683 &477.9\\
			& PWWS  & 44.47 & 11.20 & 0.609 &{\cellcolor{gray!60}} 36.3& 54.86 & 9.70 & 0.644 & {\cellcolor{gray!60}} 35.4 & 59.66 & 11.83 & 0.604 & {\cellcolor{gray!60}} 35.4\\
			& TextFooler & 78.59 & 14.67 & 0.597 &127.9 & {\cellcolor{gray!60}} 88.49 & 10.63 & 0.653 & 102.2& {\cellcolor{gray!60}} 90.08 & 12.86 & 0.623 & 100.9 \\
			& PSO   & 84.59 & 11.21 & 0.615 & 502.4& 87.21 & 9.47 & 0.652 & 390.9 & 86.03 & {\cellcolor{gray!60}} 11.32 & 0.624 & 377.7\\
			& HydraText & {\cellcolor{gray!60}} 86.12 & {\cellcolor{gray!60}} 10.79 & {\cellcolor{gray!60}} 0.702 & 272.1 & 86.83 & {\cellcolor{gray!60}} 9.36 & {\cellcolor{gray!60}} 0.724 & 244.4 & 87.08 & {\cellcolor{gray!60}} 11.31 & {\cellcolor{gray!60}} 0.714 & 262.8\\
			\midrule
			\multirow{5}[1]{*}{AG News} & GA   & 34.78 & 12.75 & 0.585 & 4326.0 & 37.74 & 12.79 & 0.687 & 3843.5 & 56.72 & 12.88 & 0.723 & 2606.3\\
			& PWWS  & 46.96 & 15.34 & 0.494 & {\cellcolor{gray!60}} 71.7 & 56.48 & 14.68 & 0.518 & {\cellcolor{gray!60}} 282.6 & 71.10 & 11.42 & 0.601 & {\cellcolor{gray!60}} 68.3\\
			& TextFooler & 57.54 & 12.87 & 0.647 & 337.9& 68.06 & 12.22 & 0.649 & 1362.5 & {\cellcolor{gray!60}} 89.96 & 11.24 & 0.708 & 190.8\\
			& PSO   & 66.38 & 13.56 & 0.605 & 5126.2 & 62.90 & 11.71 & 0.646 & 5167.7& 83.88 & 10.47 & 0.705 & 3170.9\\
			& HydraText & {\cellcolor{gray!60}} 75.04 & {\cellcolor{gray!60}} 11.72 & {\cellcolor{gray!60}} 0.676 & 1356.8 & {\cellcolor{gray!60}} 70.10 & {\cellcolor{gray!60}} 10.92 & {\cellcolor{gray!60}} 0.704 & 1362.5 &87.14 & {\cellcolor{gray!60}} 9.45 & {\cellcolor{gray!60}} 0.771 & 957.4\\
			\midrule
			\multirow{2}[3]{*}{Dataset} & \multirow{2}[3]{*}{Attack} & \multicolumn{4}{c}{BERT} & \multicolumn{4}{c}{Infersent} & \multicolumn{4}{c}{ESIM} \\
			\cmidrule(lr){3-6} \cmidrule(lr){7-10} \cmidrule(lr){11-14}	&       & {Suc. (\%)} & {Mod. (\%)} & {Sim.} &{\#Que.} & {Suc. (\%)} & {Mod. (\%)} & {Sim.} &{\#Que.}  & {Suc. (\%)} & {Mod. (\%)} & {Sim.} &{\#Que.} \\
			\midrule
			\multirow{5}[2]{*}{SNLI} & GA    & 59.76 & 11.23 & 0.511 & 195.8 & 68.80 & 11.59 & 0.534 & 127.6 & 60.40 & 11.67 & 0.526 & 177.7\\
			& PWWS  & 54.62 & 11.11 & 0.452 & {\cellcolor{gray!60}} 16.8 & 60.46 & 10.95 & 0.465 & {\cellcolor{gray!60}} 16.5 & 54.25 & 10.85 & 0.460 & {\cellcolor{gray!60}} 17.0\\
			& TextFooler & 88.90 & 11.05 & 0.547 & 68.9 & 93.97 & 11.25 & 0.592 & 63.3 & 90.29 & 11.22 & 0.572 & 68.4\\
			& PSO   & 91.94 & 11.29 & 0.444 & 140.6 & {\cellcolor{gray!60}} 95.38 & 11.23 & 0.496 & 88.7& {\cellcolor{gray!60}} 93.33 & 11.55 & 0.476 & 130.3\\
			& HydraText & {\cellcolor{gray!60}} 92.23 & {\cellcolor{gray!60}} 10.30 & {\cellcolor{gray!60}} 0.607 & 125.4 & {\cellcolor{gray!60}} 95.19 & {\cellcolor{gray!60}} 10.77 & {\cellcolor{gray!60}} 0.637 & 110.8 &{\cellcolor{gray!60}} 93.41 & {\cellcolor{gray!60}} 10.80 & {\cellcolor{gray!60}} 0.602 & 126.4\\
			\midrule
			\multirow{5}[2]{*}{MNLI} & GA    & 65.88 & 11.28 & 0.580 & 322.9 & 73.53 & 11.19 & 0.607 & 231.4 & 66.40 & 11.14 & 0.603 & 292.9\\
			& PWWS  & 57.84 & 10.26 & 0.532 & {\cellcolor{gray!60}} 20.7 & 58.75 & 10.33 & 0.564 & {\cellcolor{gray!60}} 20.8 & 57.43 & 10.32 & 0.542 & {\cellcolor{gray!60}} 20.8\\
			& TextFooler & 85.90 & 10.81 & 0.617 & 81.1 & 90.03 & 10.64 & 0.651 & 76.0 & 87.14 & 10.67 & 0.628 & 79.1\\
			& PSO   & {\cellcolor{gray!60}} 89.20 & 10.70 & 0.568 & 159.6 & 91.00 & 10.37 & 0.597 & 174.8 & {\cellcolor{gray!60}} 87.35 &10.35 & 0.583 & 170.6\\
			& HydraText & {\cellcolor{gray!60}} 89.18 & {\cellcolor{gray!60}} 10.00 & {\cellcolor{gray!60}} 0.680 & 139.1 & {\cellcolor{gray!60}} 91.18 & {\cellcolor{gray!60}} 10.11 & {\cellcolor{gray!60}} 0.707 & 143.9 & {\cellcolor{gray!60}} 87.31 & {\cellcolor{gray!60}} 9.82 & {\cellcolor{gray!60}} 0.688 & 141.3\\
			\midrule
			\midrule
			\multicolumn{14}{c}{\textbf{Decision-based Setting}} \\
			\midrule
			\multirow{2}[3]{*}{Dataset} & \multirow{2}[3]{*}{Attack} & \multicolumn{4}{c}{BERT} & \multicolumn{4}{c}{WordLSTM} & \multicolumn{4}{c}{WordCNN} \\
			\cmidrule(lr){3-6} \cmidrule(lr){7-10} \cmidrule(lr){11-14}  &       & {Suc. (\%)} & {Mod. (\%)} & {Sim.} &{\#Que.} & {Suc. (\%)} & {Mod. (\%)} & {Sim.} &{\#Que.}  & {Suc. (\%)} & {Mod. (\%)} & {Sim.} &{\#Que.} \\
			\midrule
			\multirow{2}[1]{*}{IMDB} & GADe  & 99.72 & 5.68 & 0.813 & 1740.0 & 99.27 & 5.76 & 0.809 & 1664.8 & 99.46 & 5.78 & 0.820 & 1612.6\\
			& HydraText & {\cellcolor{gray!60}} 99.95 & {\cellcolor{gray!60}} 3.95 & {\cellcolor{gray!60}} 0.853 & {\cellcolor{gray!60}} 1229.1 & {\cellcolor{gray!60}} 99.66 & {\cellcolor{gray!60}} 4.16 & {\cellcolor{gray!60}} 0.846 & {\cellcolor{gray!60}} 1519.1 & {\cellcolor{gray!60}} 99.76 & {\cellcolor{gray!60}} 4.18 & {\cellcolor{gray!60}} 0.860 & {\cellcolor{gray!60}} 1412.0\\
			\midrule
			\multirow{2}[2]{*}{MR} & GADe  & 89.06 & 10.56 & 0.680 & {\cellcolor{gray!60}} 544.0 & 89.39 & 10.78 & 0.675 & {\cellcolor{gray!60}} 485.0 & 91.78 & 10.72 & 0.680 & {\cellcolor{gray!60}} 478.2\\
			& HydraText & {\cellcolor{gray!60}} 93.88 & {\cellcolor{gray!60}} 9.06 & {\cellcolor{gray!60}} 0.748 & 954.9 & {\cellcolor{gray!60}} 93.73 & {\cellcolor{gray!60}} 9.19 & {\cellcolor{gray!60}} 0.762 & 772.5 & {\cellcolor{gray!60}} 95.30 & {\cellcolor{gray!60}} 9.18 & {\cellcolor{gray!60}} 0.772 & 718.8\\
			\midrule
			\multirow{2}[2]{*}{AG News} & GADe  & 91.32 & 11.87 & 0.693 & {\cellcolor{gray!60}} 3425.0 &89.06 & 12.09 & 0.710 & {\cellcolor{gray!60}} 3393.3 & 95.82 & 10.21 & 0.776 & {\cellcolor{gray!60}} 2226.5\\
			& HydraText & {\cellcolor{gray!60}} 96.32 & {\cellcolor{gray!60}} 8.69 & {\cellcolor{gray!60}} 0.729 & 3803.6 & {\cellcolor{gray!60}} 94.60 & {\cellcolor{gray!60}} 9.08 & {\cellcolor{gray!60}} 0.746 & 3537.0 & {\cellcolor{gray!60}} 98.22 & {\cellcolor{gray!60}} 7.68 & {\cellcolor{gray!60}} 0.815 & 2576.0\\
			\midrule
			\multirow{2}[4]{*}{Dataset} & \multirow{2}[4]{*}{Attack} & \multicolumn{4}{c}{BERT} & \multicolumn{4}{c}{Infersent} & \multicolumn{4}{c}{ESIM} \\
			\cmidrule(lr){3-6} \cmidrule(lr){7-10} \cmidrule(lr){11-14}	&       & {Suc. (\%)} & {Mod. (\%)} & {Sim.} &{\#Que.} & {Suc. (\%)} & {Mod. (\%)} & {Sim.} &{\#Que.}  & {Suc. (\%)} & {Mod. (\%)} & {Sim.} &{\#Que.} \\
			\midrule
			\multirow{2}[2]{*}{SNLI} & GADe  & 88.98 & 11.45 & 0.493 & {\cellcolor{gray!60}} 243.1& 93.13 & 11.56 & 0.522 & {\cellcolor{gray!60}} 181.0 & 89.90 & 12.24 & 0.512 & {\cellcolor{gray!60}} 259.8\\
			& HydraText & {\cellcolor{gray!60}} 94.94 & {\cellcolor{gray!60}} 9.88 & {\cellcolor{gray!60}} 0.638 & 481.3 & {\cellcolor{gray!60}} 98.38 & {\cellcolor{gray!60}} 9.91 & {\cellcolor{gray!60}} 0.679 & 257.6 & {\cellcolor{gray!60}} 96.27 & {\cellcolor{gray!60}} 10.46 & {\cellcolor{gray!60}} 0.662 & 524.1\\
			\midrule
			\multirow{2}[2]{*}{MNLI} & GADe  & 91.54 & 10.88 & 0.585& {\cellcolor{gray!60}} 278.8 & 91.45 & 11.07 & 0.619 & {\cellcolor{gray!60}} 308.6 & 93.03 & 10.67 & 0.604 & {\cellcolor{gray!60}} 253.7\\
			& HydraText & {\cellcolor{gray!60}} 96.60 & {\cellcolor{gray!60}} 9.55 & {\cellcolor{gray!60}} 0.712 & 579.3 & {\cellcolor{gray!60}} 96.12 & {\cellcolor{gray!60}} 9.59 & {\cellcolor{gray!60}} 0.751 & 1230.4 & {\cellcolor{gray!60}} 96.39 & {\cellcolor{gray!60}} 9.41 & {\cellcolor{gray!60}} 0.727 & 833.6\\
			\bottomrule
	\end{tabular}}
	\label{tab:main_results}
\end{table*}

\begin{table}[tbp]
	\centering
	\caption{The grammaticality and fluency results of the AEs crafted by the compared approaches.
		Grammaticality is measured by grammatical error increase rate (EIR).
		Fluency is measured by the perplexity (PPL) of GPT-2.
		{\color{black}On each dataset, the test result is indicated in gray if it was not significantly different from the best result (according to a Wilcoxon signed-rank test with significance level $p= 0.05$).}
		For both metrics, the lower the better.}
	\scalebox{0.68}{
		\begin{tabular}{ccS[table-format=2.2]S[table-format=3.1]S[table-format=2.2]S[table-format=3.1]S[table-format=2.2]S[table-format=3.1]}
			\toprule
			\multicolumn{8}{c}{\textbf{Score-based Setting}} \\
			\midrule
			\multirow{2}[4]{*}{Dataset} & \multirow{2}[4]{*}{Attack} & \multicolumn{2}{c}{BERT} & \multicolumn{2}{c}{WordLSTM} & \multicolumn{2}{c}{WordCNN} \\
			\cmidrule(lr){3-4} \cmidrule(lr){5-6} \cmidrule(lr){7-8}	&       & {EIR (\%)}   & {PPL}   & {EIR (\%)}   & {PPL}   & {EIR (\%)}   & {PPL} \\
			\midrule
			\multirow{5}[2]{*}{IMDB} & GA    & 5.66 & 175.3 & 5.52 & 166.5 & 5.92 & 166.1 \\
			& PWWS  & 2.94 & 141.5 & 2.12 & 129.5 & 2.92 & 126.5 \\
			& TF    & 3.23 & 132.3 & 3.47 & 120.7 & 4.17 & 119.1 \\
			& PSO   & {\cellcolor{gray!60}} 0.97 & 127.7 & {\cellcolor{gray!60}} 1.12 & 123.3 & 1.89 & 118.9 \\
			& HydraText & 1.21 & {\cellcolor{gray!60}} 113.7 & 1.18 & {\cellcolor{gray!60}} 110.2 & {\cellcolor{gray!60}} 1.61 & {\cellcolor{gray!60}} 107.7 \\
			\midrule
			\multirow{5}[2]{*}{MR} & GA    & 3.61 & 412.4 & 5.45 & 530.4 & 5.85 & 528.6 \\
			& PWWS  & 3.44 & 569.2 & 5.72 & 467.9 & 6.25 & 610.4 \\
			& TF    & 4.95 & 407.9 & 5.77 & 409.8 & 5.68 & 468.8 \\
			& PSO   & {\cellcolor{gray!60}} {\cellcolor{gray!60}}1.20 & 574   & {\cellcolor{gray!60}} 1.27 & 464.8 & 1.92 & 654.3 \\
			& HydraText & {\cellcolor{gray!60}} 1.18 & {\cellcolor{gray!60}} 333.4 & 1.31 & {\cellcolor{gray!60}} 378.8 & {\cellcolor{gray!60}} 1.24 & {\cellcolor{gray!60}} 375.6 \\
			\midrule
			\multirow{5}[2]{*}{AG News} & GA    & 4.41 & 355.2 & 4.95 & 338.1 & 5.20 & 365.5 \\
			& PWWS  & 5.85 & 356.2 & 4.78 & 356.1 & 5.29 & 330.6 \\
			& TF    & 3.12 & 360.3 & 4.30 & 339.4 & 3.85 & 319.7 \\
			& PSO   & 0.67 & 387   & 1.48 & 403.5 & 1.54 & 367.1 \\
			& HydraText & {\cellcolor{gray!60}} 0.54 & {\cellcolor{gray!60}} 218.6 & {\cellcolor{gray!60}} 1.20 & {\cellcolor{gray!60}} 231.3 & {\cellcolor{gray!60}} 1.36 & {\cellcolor{gray!60}} 245.8 \\
			\midrule
			\multirow{2}[4]{*}{Dataset} & \multirow{2}[4]{*}{Attack} & \multicolumn{2}{c}{BERT} & \multicolumn{2}{c}{Infersent} & \multicolumn{2}{c}{ESIM} \\
			\cmidrule(lr){3-4} \cmidrule(lr){5-6} \cmidrule(lr){7-8}  &       & {EIR (\%)}   & {PPL}   & {EIR (\%)}   & {PPL}   & {EIR (\%)}   & {PPL} \\
			\midrule
			\multirow{5}[2]{*}{SNLI} & GA    & 14.19 & 328.6 & 17.77 & 353.4 & 17.73 & 363.7 \\
			& PWWS  & 14.06 & 306.6 & 12.41 & 342   & 14.36 & 327.4 \\
			& TF    & 25.81 & 329.7 & 28.01 & 310.9 & 25.46 & 328.1 \\
			& PSO   & 8.75 & 316.5 & {\cellcolor{gray!60}} 6.13 & 319.5 & 7.28 & 327.5 \\
			& HydraText & {\cellcolor{gray!60}} 7.60 & {\cellcolor{gray!60}} 300.5 & 7.34 & {\cellcolor{gray!60}} 299.3 & {\cellcolor{gray!60}} 6.47 & {\cellcolor{gray!60}} 310.2 \\
			\midrule
			\multirow{5}[2]{*}{MNLI} & GA    & 15.65 & 419.6 & 16.17 & 410.2 & 17.53 & 440.4 \\
			& PWWS  & 5.09 & {\cellcolor{gray!60}} 320.5 & 7.07 & {\cellcolor{gray!60}} 346.1 & 5.84 & {\cellcolor{gray!60}} 325.7 \\
			& TF    & 13.61 & 400.9 & 15.07 & 387.2 & 17.35 & 391.9 \\
			& PSO   & {\cellcolor{gray!60}} 1.91 & 409.5 & {\cellcolor{gray!60}} 2.16 & 399.9 & {\cellcolor{gray!60}} 2.20 & 411.4 \\
			& HydraText & 2.35 & 399.1 & 2.61 & 387.8 & 2.32 & 390.8 \\
			\midrule
			\midrule
			\multicolumn{8}{c}{\textbf{Decision-based Setting}} \\
			\midrule
			
			\multirow{2}[4]{*}{Dataset} & \multirow{2}[4]{*}{Attack} & \multicolumn{2}{c}{BERT} & \multicolumn{2}{c}{WordLSTM} & \multicolumn{2}{c}{WordCNN} \\
			\cmidrule(lr){3-4} \cmidrule(lr){5-6} \cmidrule(lr){7-8} &       & {EIR (\%)}   & {PPL}   & {EIR (\%)}   & {PPL}   & {EIR (\%)}   & {PPL} \\
			\midrule
			\multirow{2}[2]{*}{IMDB} & GADe  & 3.67 & 126.7 & 3.67 & 125.6 & 4.10 & 125.2 \\
			& HydraText & {\cellcolor{gray!60}} 2.43 & {\cellcolor{gray!60}} 111.7 & {\cellcolor{gray!60}} 2.61 & {\cellcolor{gray!60}} 110.6 & {\cellcolor{gray!60}} 3.15 & {\cellcolor{gray!60}} 110.5 \\
			\midrule
			\multirow{2}[2]{*}{MR} & GADe  & 4.88 & 432.9 & 4.67 & 488.3 & 6.32 & 603.7 \\
			& HydraText & {\cellcolor{gray!60}} 3.80 & {\cellcolor{gray!60}} 395.6 & {\cellcolor{gray!60}} 3.62 & {\cellcolor{gray!60}} 452.1 & {\cellcolor{gray!60}} 4.43 & {\cellcolor{gray!60}} 499.1 \\
			\midrule
			\multirow{2}[2]{*}{AG News} & GADe  & 3.61 & 350.2 & 4.52 & 345.4 & 4.28 & 318.6 \\
			& HydraText & {\cellcolor{gray!60}} 2.21 & {\cellcolor{gray!60}} 292.3 & {\cellcolor{gray!60}} 3.41 & {\cellcolor{gray!60}} 286.8 & {\cellcolor{gray!60}} 3.82 & {\cellcolor{gray!60}} 270.4 \\
			\midrule
			\multirow{2}[4]{*}{Dataset} & \multirow{2}[4]{*}{Attack} & \multicolumn{2}{c}{BERT} & \multicolumn{2}{c}{Infersent} & \multicolumn{2}{c}{ESIM} \\
			\cmidrule(lr){3-4} \cmidrule(lr){5-6} \cmidrule(lr){7-8} &       & {EIR (\%)}   & {PPL}   & {EIR (\%)}   & {PPL}   & {EIR (\%)}   & {PPL} \\
			\midrule
			\multirow{2}[2]{*}{SNLI} & GADe  & 24.82 & 364.1 & 30.14 & 361.1 & 27.91 & 382.4 \\
			& HydraText & {\cellcolor{gray!60}} 19.33 & {\cellcolor{gray!60}} 294.6 & {\cellcolor{gray!60}} 26.80 & {\cellcolor{gray!60}} 295.7 & {\cellcolor{gray!60}} 19.24 & {\cellcolor{gray!60}} 306.9 \\
			\midrule
			\multirow{2}[2]{*}{MNLI} & GADe  & 14.36 & 438   & 16.54 & 460.2 & 16.68 & 437.7 \\
			& HydraText & {\cellcolor{gray!60}} 11.55 & {\cellcolor{gray!60}} 377.9 & {\cellcolor{gray!60}} 13.11 & {\cellcolor{gray!60}} 385.6 & {\cellcolor{gray!60}} 12.93 & {\cellcolor{gray!60}} 373.3 \\
			\bottomrule
		\end{tabular}
	}
	\label{tab:complete_quality}%
\end{table}%

\subsection{Results and Analysis}
\label{sec:results}
Table~\ref{tab:main_results} presents the attack performance of HydraText and the baselines in terms of success rate, modification rate, semantic similarity, and query number.
{
\color{black}
Additionally, on each dataset, we utilized the Wilcoxon signed-rank test (with significance level $p= 0.05$) to determine whether the differences  between the test results were significant.
}

The first observation from these results is that compared to the baselines, HydraText performed consistently better in crafting imperceptible attacks.
Specifically, on \textit{all} datasets, it \textcolor{black}{significantly outperformed} \textit{all} the baselines in terms of modification rate and semantic similarity, indicating the necessity of directly optimizing attack imperceptibility during the search for AEs.
More importantly, such performance was achieved without compromising on attack success rates.
In general, HydraText could achieve quite competitive attack success rates compared to the baselines.
\textcolor{black}{For example, in score-based setting, on 12 out of 15 datasets, HydraText achieved the highest success rates, which is the most among all the approaches.}
Notably, when attacking BERT/WordLSTM/WordCNN on the IMDB dataset, HydraText achieved 100/99.90/99.90 success rates.
In decision-based setting, on all the 15 datasets, HydraText achieved \textcolor{black}{significantly} higher success rates than GADe.
We speculate that the reason for the strong attacking ability of HydraText is its integration of multiple mutation operators, enabling the algorithm to explore a much larger neighborhood solution space than the baselines which generally only use one type of mutation.

In terms of query number, as expected, the simple greedy heuristic-based approaches PWWS and TextFooler consumed much fewer model queries than those population-based approaches.
From the perspective of heuristics, the former are construction heuristics which are fast but may not be effective, while the latter are search heuristics which are more time-consuming but can find better solutions.
Indeed, in Table~\ref{tab:main_results} PWWS and TextFooler often achieved \textcolor{black}{significantly lower} success rates than HydraText.
For example, compared to PWWS, HydraText generally obtained much higher success rates, typically by around 20\%-40\%.
\textcolor{black}{Compared to the more state-of-the-art TextFooler, HydraText still showed significant advantages in terms of success rates on 10 out of 15 datasets.}
In score-based setting, on 14 out of 15 datasets, HydraText consumed fewer queries than GA and PSO, making it the most query-efficient population-based approach.
In decision-based setting, HydraText required 19.5\% more queries than GADe on average.
Such results are conceivable because GADe searches from a randomly initialized solution which is already a successful AE yet far from the original input, while HydraText searches from the original input.

In summary, thanks to its multi-objectivization mechanism, HydraText indeed could craft more imperceptible textual attacks, without compromising on attack success rate.
On the other hand,  solving the induced multi-objective optimization problem would also consume a considerable number of queries, especially when compared to greedy heuristic-based approaches.
In future work, it is possible to improve the query efficiency of HydraText by leveraging word importance (like TextFooler) or learning-based models that directly predict a solution~\cite{LiuZTY23}.

Table~\ref{tab:complete_quality} presents the grammaticality and fluency results of the AEs crafted by the compared approaches.
Overall, the AEs crafted by HydraText have better quality than the ones crafted by the baselines.
On all the 15 datasets, HydraText could achieve at least top-2 performance regarding either EIR or PPL.
Specifically, in decision-based setting, it always \textcolor{black}{significantly} outperformed GADe in both EIR and PPL.
In score-based setting, HydraText achieved the lowest EIR and the lowest PPL, on 8 and 12 out of 15 datasets, respectively.
In comparison, for the second-best approach PSO, the corresponding numbers are \textcolor{black}{8} and 0, respectively.
The above results indicate that compared to the baselines, HydraText could craft AEs that are less detectable by automatic detectors.

\begin{table}[tbp]
	\centering
	\caption{Human evaluation results of validity (percentage of valid AEs) and naturality (average naturality score) of the crafted AEs on attacking BERT on the MR dataset.
		For each metric, the best value is indicated in gray.
		For both metrics, the higher the better.}
	\scalebox{0.8}{
		\begin{tabular}{cS[table-format=2.2]S[table-format=1.2]}
			\toprule
			Attack & \multicolumn{1}{l}{Percentage of Validity AEs (\%)} & \multicolumn{1}{l}{Average Naturality Score} \\
			\midrule
			GADe  & 62.96 & 2.96 \\
			PSO   & 66.67  & 3.18 \\
			HydraText &  {\cellcolor{gray!60}} 77.78 & {\cellcolor{gray!60}} 3.59 \\
			\bottomrule
	\end{tabular}}
	\label{tab:human_eval}%
\end{table}%

\subsection{Human Evaluation}
\label{sec:human_eval}
To assess how likely the attacks are to be perceived by humans, we performed a human evaluation study on 400 AEs crafted by HydraText and the best competitors PSO and GADe for attacking BERT on the MR dataset, respectively.
With questionnaire survey (see Appendix~\ref{appendix:humaneval} for an example), we asked five movie fans (volunteers) to make a binary sentiment classification (i.e., labeling it as ``positive'' or ``negative''), and give a naturality score chosen from $\{1,2,3,4,5\}$ indicating machine-generated, more like machine-generated, not sure, more like human-written and human-written, respectively.
The final sentiment labels were determined by majority voting, and the final naturality scores were determined by averaging.
Note that an AE is valid if its human prediction (true label) is the same as the original input.
For example, supposing an attacker changes ``good'' to ``bad'' in an input sentence ``the movie is good'' to attack a sentiment classification model, although the target model alters its prediction result, this attack is invalid. 

Table~\ref{tab:human_eval} presents the final results of human evaluation, in terms of the percentage of valid AEs and the average naturality score.
Compared to the baselines, HydraText performed better in crafting valid and natural AEs.
Specifically, 77.78\% of the AEs crafted by it were judged by humans to be valid, which was at least 10\% higher than the baselines.
Notably, it achieved an average naturality socre of 3.59.
Recall that the naturality scores of 3 and 4 indicate not-sure and more like human-written, respectively.
Such results demonstrate that between not-sure and more like human-written, human evaluators tend to consider the AEs crafted by HydratText as the latter.
In contrast, the AEs crafted by PSO and GADe are considered more likely to be the former.

\begin{table*}[tbp]
	\centering
	\caption{Some AEs crafted by HydraText, PSO and GADe when attacking BERT on the MR dataset.}
	\scalebox{0.68}{
		\begin{tabular}{c}
			\toprule[.2em]
			\multicolumn{1}{c}{MR Example 1} \\
			\midrule[.1em]
			\multicolumn{1}{l}{\textbf{Original Input} (Prediction=Positive): the draw for Big Bad Love is a solid performance by Arliss Howard} \\
			\midrule[.1em]
			\multicolumn{1}{l}{\textbf{HydraText} (Prediction=Negative): the draw for Big Bad Love is a \color{red}{dependable} performance by Arliss Howard} \\
			\midrule
			\multicolumn{1}{l}{\textbf{GADe} (Prediction=Negative): the draw for \color{red}{Largest} Bad Love is a \color{red}{sturdy} \color{red}{discharging} by Arliss Howard} \\
			\midrule
			\multicolumn{1}{l}{\textbf{PSO} (Prediction=Negative):the draw for Big Bad Love is a \color{red}{much} performance by Arliss Howard} \\
			\midrule[.2em]
			\multicolumn{1}{c}{MR Example 2} \\
			\midrule[.1em]
			\multicolumn{1}{l}{\shortstack[l]{\textbf{Original Input} (Prediction=Positive): Insomnia does not become one of those rare remakes to eclipse the\\ original, but it doesn't disgrace it, either.}} \\
			\midrule[.1em]
			\multicolumn{1}{l}{\shortstack[l]{\textbf{HydraText} (Prediction=Negative): Insomnia does not become one of those rare remakes to eclipse the\\ original, but it doesn't \color{red}{humiliate} it, either.}} \\
			\midrule
			\multicolumn{1}{l}{\shortstack[l]{\textbf{GADe} (Prediction=Negative): Insomnia does not become one of those rare remakes to eclipse the\\ \color{red}{special}, but it doesn't \color{red}{infamy} it, either.}} \\
			\midrule
			\multicolumn{1}{l}{\shortstack[l]{\textbf{PSO} (Prediction=Negative): Insomnia does not become one of those \color{red}{little} remakes to eclipse the\\ \color{red}{special}, but it doesn't \color{red}{humiliate} it, either.}} \\
			\midrule[.2em]
			\multicolumn{1}{c}{MR Example 3} \\
			\midrule[.1em]
			\multicolumn{1}{l}{\shortstack[l]{\textbf{Original Input} (Prediction=Negative): Paul Bettany is good at being the ultra violent gangster wannabe,\\ but the movie is certainly not number 1}} \\
			\midrule[.1em]
			\multicolumn{1}{l}{\shortstack[l]{\textbf{HydraText} (Prediction=Positive):  Paul Bettany is \color{red}{skillful} at being the ultra violent gangster wannabe,\\ but the movie is \color{red}{surely} not number 1}} \\
			\midrule
			\multicolumn{1}{l}{\shortstack[l]{\textbf{GADe} (Prediction=Positive):  Paul Bettany is good at being the ultra violent gangster wannabe,\\ but the movie is \color{red}{soberly} \color{red}{either} \color{red}{figures} 1}} \\
			\midrule
			\multicolumn{1}{l}{\shortstack[l]{\textbf{PSO} (Prediction=Positive):  Paul Bettany is good at being the ultra violent gangster wannabe,\\ but the movie is certainly not \color{red}{variable} 1}} \\
			\midrule[.2em]
			\multicolumn{1}{c}{MR Example 4} \\
			\midrule[.1em]
			\multicolumn{1}{l}{\textbf{Original Input} (Prediction=Negative): as saccharine as it is disposable} \\
			\midrule[.1em]
			\multicolumn{1}{l}{\textbf{HydraText} (Prediction=Positive): as saccharine as it is \color{red}{expendable}} \\
			\midrule
			\multicolumn{1}{l}{\textbf{GADe} (Prediction=Positive): as saccharine as it is \color{red}{usable}} \\
			\midrule
			\multicolumn{1}{l}{\textbf{PSO} (Prediction=Positive): as saccharine as it is \color{red}{convenient}} \\
			\midrule[.2em]
			\multicolumn{1}{c}{MR Example 5} \\
			\midrule[.1em]
			\multicolumn{1}{l}{\textbf{Original Input} (Prediction=Positive): a terrific date movie, whatever your orientation} \\
			\midrule[.1em]
			\multicolumn{1}{l}{\textbf{HydraText} (Prediction=Negative): a \color{red}{glamorous} date movie, whatever your orientation} \\
			\midrule
			\multicolumn{1}{l}{\textbf{GADe} (Prediction=Negative): a \color{red}{sumptuous yesterday} movie, whatever your orientation} \\
			\midrule
			\multicolumn{1}{l}{\textbf{PSO} (Prediction=Negative): a \color{red}{heavy} date movie, whatever your orientation} \\
			\midrule[.2em]
			\multicolumn{1}{c}{MR Example 6} \\
			\midrule[.1em]
			\multicolumn{1}{l}{\shortstack[l]{\textbf{Original Input} (Prediction=Positive): a thoroughly entertaining comedy that uses Grant's own twist of acidity to\\ prevent itself from succumbing to its own bathos}} \\
			\midrule[.1em]
			\multicolumn{1}{l}{\shortstack[l]{\textbf{HydraText} (Prediction=Negative): a thoroughly \color{red}{comical} comedy that uses Grant's own twist of acidity to\\ \color{red}{hinder} itself from succumbing to its own bathos}} \\
			\midrule
			\multicolumn{1}{l}{\shortstack[l]{\textbf{GADe} (Prediction=Negative): a thoroughly \color{red}{funnier} comedy that uses Grant's own twist of acidity to\\ \color{red}{hinder} itself from succumbing to its own bathos}} \\
			\midrule
			\multicolumn{1}{l}{\shortstack[l]{\textbf{PSO} (Prediction=Negative): a \color{red}{terribly} entertaining comedy that uses Grant's own twist of acidity to\\ \color{red}{block} itself from \color{red}{surrendering} to its own bathos}} \\
			\midrule[.2em]
			\multicolumn{1}{c}{MR Example 7} \\
			\midrule[.1em]
			\multicolumn{1}{l}{\shortstack[l]{\textbf{Original Input} (Prediction=Negative): With virtually no interesting elements for an audience to focus on,\\ Chelsea Walls is a triple-espresso endurance challenge}} \\
			\midrule[.1em]
			\multicolumn{1}{l}{\shortstack[l]{\textbf{HydraText} (Prediction=Positive): With \color{red}{almost} no interesting elements for an audience to focus on,\\ Chelsea Walls is a triple-espresso endurance challenge}} \\
			\midrule
			\multicolumn{1}{l}{\shortstack[l]{\textbf{GADe} (Prediction=Positive): With virtually no interesting elements for an audience to focus on,\\ Chelsea Walls is a triple-espresso \color{red}{vitality} challenge}} \\
			\midrule
			\multicolumn{1}{l}{\shortstack[l]{\textbf{PSO} (Prediction=Positive): With virtually no interesting elements for an \color{red}{viewer} to focus on,\\ Chelsea Walls is a \color{red}{inestimable}-espresso endurance challenge}} \\
			\bottomrule[.2em]
	\end{tabular}}
	\label{tab:cases}
\end{table*}

Combining the results in Table~\ref{tab:main_results}, Table~\ref{tab:complete_quality} and Table~\ref{tab:human_eval}, it can be concluded that among all the compared approaches, HydratText could craft the most imperceptible textual attacks to either automatic detectors or humans .
Table~\ref{tab:cases} lists some AEs crafted by HydraText and the best competitors PSO and GADe.

\begin{table}[tbp]
	\centering
	\caption{Evaluation Results on transferability of the crafted AEs.
		``B $\Rightarrow$ WL'' means transferring from BERT to WordLSTM and vice versa;
		the change in the model's accuracy on the AEs compared to its accuracy on the original input is presented.}
	\scalebox{0.8}{
		\begin{tabular}{ccccccc}
			\toprule
			\multirow{2}[4]{*}{Transfer} & \multicolumn{3}{c}{Decision-based Setting} & \multicolumn{3}{c}{Score-based Setting} \\
			\cmidrule(lr){2-4}\cmidrule(lr){5-7}
			& IMDB  & MR    & AGNews & IMDB  & MR    & AGNews \\
			\cmidrule(lr){1-7}
			B $\Rightarrow$ WL & -12.64 & -19.39 & -7.06 & -15.49 & -21.00 & -7.11 \\
			WL $\Rightarrow$ B & -9.07 & -13.45 & -9.07 & -12.70 & -16.86 & -12.70 \\
			\bottomrule
	\end{tabular}}
	\label{tab:transfer}
\end{table}
\begin{table}[tbp]
	\centering
	\caption{The evaluation results of using HydraText to attack BERT before (Original) and after adversarial training (Adv. T), in terms of attack success rate (S), average modification rate (M) and average semantic similarity (Sim.). For each metric, the best value is indicated in gray.}
	\scalebox{0.8}{
		\begin{tabular}{cS[table-format=3.2]S[table-format=1.2]S[table-format=1.3]S[table-format=3.2]S[table-format=1.2]S[table-format=1.3]}
			\toprule
			\multirow{2}[4]{*}{Model} & \multicolumn{3}{c}{IMDB} & \multicolumn{3}{c}{MR} \\
			\cmidrule(lr){2-4}\cmidrule(lr){5-7}
			& {S (\%)} & {M (\%)} &  {Sim.} & {S (\%)} & {M (\%)}  & {Sim.} \\
			\midrule
			Original & {\cellcolor{gray!60}} 100.00 & {\cellcolor{gray!60}} 3.74 & {\cellcolor{gray!60}} 0.851 & {\cellcolor{gray!60}} 86.12 & {\cellcolor{gray!60}} 9.37 & {\cellcolor{gray!60}} 0.720 \\   
			Adv. T & 96.12 & 5.88 & 0.793 & 82.50 & 10.47 & 0.691 \\
			\bottomrule
	\end{tabular}}
	\label{tab:adv_train}
\end{table}

\subsection{Transferability and Adversarial Training}
The transferability of an AE refers to its ability to attack other unseen models \cite{GoodfellowSS15}.
We evaluated the transferability on the IMDB, MR and AG News datasets, by transferring AEs between BERT and WordLSTM.
More specifically, on each dataset, we used WordLSTM to classify the AEs crafted for attacking BERT (denoted as ``B $\Rightarrow$ WL''), and vice versa.
Table~\ref{tab:transfer} demonstrates that in either score-based or decision-based setting, on each dataset, no matter transferred in either direction, the AEs crafted by HydraText could always result in a significant decrease in the accuracy of the tested model, indicating their good transferability.

By incorporating AEs into the training process, adversarial training is aimed at improving the robustness of target models.
We generated AEs by using score-based HydraText to attack BERT on 8\% instances randomly sampled from the training sets of IMDB dataset, and included them into the training sets and retrained BERT.
We then attacked the retrained model by score-based HydraText.
The above procedure is repeated for the MR dataset.
As presented in Table~\ref{tab:adv_train}, it can be found that on the retained models, the attack performance significantly degrades in terms of all the three evaluation metrics.
This indicates that the AEs crafted by HydraText can indeed bring robustness improvement to the target models.

\begin{table*}[htbp]
	\centering
	\caption{The attack performance achieved by the compared approaches when crafting targeted attacks, in terms of attack success rate (Suc.), average modification rate (Mod.), average semantic similarity (Sim.), and average query number (\#Que.), in score-based and decision-based settings.
	{\color{black}On each dataset, the test result in terms of attack success rate, average modification rate, and average semantic similarity is indicated in gray if it was not significantly different from the best result (according to a Wilcoxon signed-rank test with significance level $p= 0.05$).}
		Note for success rate and semantic similarity, the higher the better; while for modification rate and query number, the lower the better.}
	\scalebox{0.68}{
		\begin{tabular}{ccS[table-format=2.2]S[table-format=2.2]S[table-format=1.3]S[table-format=4.1]S[table-format=2.2]S[table-format=2.2]S[table-format=1.3]S[table-format=4.1]S[table-format=2.2]S[table-format=2.2]S[table-format=1.3]S[table-format=4.1]}
			\toprule
			\multicolumn{14}{c}{\textbf{Score-based Setting}} \\
			\midrule
			\multirow{2}[4]{*}[2pt]{Dataset} & \multirow{2}[4]{*}[2pt]{Attack} & \multicolumn{4}{c}{BERT} & \multicolumn{4}{c}{Infersent} & \multicolumn{4}{c}{ESIM} \\
			\cmidrule(lr){3-6} \cmidrule(lr){7-10} \cmidrule(lr){11-14}   &       &  {Suc. (\%)} & {Mod. (\%)} & {Sim.} &{\#Que.}  &  {Suc. (\%)} & {Mod. (\%)} & {Sim.} &{\#Que.} &  {Suc. (\%)} & {Mod. (\%)} & {Sim.} &{\#Que.}\\
			\midrule
			\multirow{5}[2]{*}[2pt]{SNLI} & GA    & 37.44 & 12.29 & 0.489 & 375.5& 44.18 & 12.74 & 0.517 & 271.8 & 32.84 & 13.03 & 0.498 & 379.8\\
			& PWWS  & 30.25 & 11.39 & 0.424 & {\cellcolor{gray!60}} 17.5 & 31.09 & 11.43 & 0.435 & {\cellcolor{gray!60}} 17.4 & 26.52 & 11.97 & 0.425 & {\cellcolor{gray!60}} 17.7\\
			& TextFooler & 56.02 & 11.67 & 0.490 & 119.5 & 59.86 & 11.94 & 0.538 & 115.3 &  52.51 & 12.10 & 0.524 & 126.0 \\
			& PSO   & 72.79 & 11.82 & 0.455 &445.6 & {\cellcolor{gray!60}}74.29 & 11.52 & 0.471 & 421.3 & {\cellcolor{gray!60}} 70.45 & 12.37 & 0.448 & 525.9\\
			& HydraText & {\cellcolor{gray!60}} 73.90 & {\cellcolor{gray!60}} 11.82 & {\cellcolor{gray!60}} 0.554 & 244.4 & {\cellcolor{gray!60}} 74.66 & {\cellcolor{gray!60}} 11.26 & {\cellcolor{gray!60}} 0.578 & 234.2 & {\cellcolor{gray!60}} 70.93 & {\cellcolor{gray!60}} 11.79 & {\cellcolor{gray!60}} 0.548 & 262.1 \\
			\midrule
			\multicolumn{14}{c}{{\textbf{Decision-based Setting}}} \\
			\midrule
			\multirow{2}[4]{*}[2pt]{Dataset} & \multirow{2}[4]{*}[2pt]{Attack} & \multicolumn{4}{c}{BERT} & \multicolumn{4}{c}{Infersent} & \multicolumn{4}{c}{ESIM} \\
			\cmidrule(lr){3-6} \cmidrule(lr){7-10} \cmidrule(lr){11-14}  &       &  {Suc. (\%)} & {Mod. (\%)} & {Sim.} &{\#Que.}  &  {Suc. (\%)} & {Mod. (\%)} & {Sim.} &{\#Que.} &  {Suc. (\%)} & {Mod. (\%)} & {Sim.} &{\#Que.} \\
			\midrule
			\multirow{2}[2]{*}[2pt]{SNLI} & GADe  & 70.82 & 12.96 & 0.473 & {\cellcolor{gray!60}} 1800.6 & 73.09 & 12.86 & 0.507 & {\cellcolor{gray!60}} 1657.7 & 68.07 & 13.70 & 0.494 & {\cellcolor{gray!60}} 1967.4\\
			& HydraText & {\cellcolor{gray!60}} 82.57 & {\cellcolor{gray!60}} 11.71 & {\cellcolor{gray!60}} 0.571 & 2680.7 & {\cellcolor{gray!60}} 84.27 & {\cellcolor{gray!60}} 11.44 & {\cellcolor{gray!60}} 0.610 & 4192.0 &{\cellcolor{gray!60}} 79.12 & {\cellcolor{gray!60}} 12.16 & {\cellcolor{gray!60}} 0.601 & 3823.6\\
			\bottomrule
		\end{tabular}%
	}
	\label{tab:targeted}
\end{table*}

\subsection{Targeted Attack}
\label{sec:targeted}
HydraText can also be applied to craft \textit{targeted} attacks, where the goal is to change the target model's prediction to a \textit{particular} wrong class.
We assessed the performance of HydraText and the baselines on the SNLI dataset.
Recall that each instance in SNLI comprises a premise-hypothesis sentence pair and is labeled one of three relations including entailment (ent.), contradiction (con.) and neutral (neu.).
Let $\mathbf{c}, \mathbf{c}_{tar}, \mathbf{c}_{adv}$ denote the original label, the target label, and the label of the AE $\mathbf{x}$, respectively.
Following \cite{ZangQYLZLS20}, we set $\mathbf{c}_{tar}$ as follows ($r$ is a random number sampled from $[0,1]$):
\begin{equation}
	\label{eq:score-based-target}
	\mathbf{c}_{tar}=\left\{\begin{array}{lll}
		\text{ent.}, & (\mathbf{c}=\text{con.})\lor(\mathbf{c}=\text{neu.}\land r\leq 0.5)\\
		\text{con.}, & (\mathbf{c}=\text{ent.})\lor(\mathbf{c}=\text{neu.}\land r>0.5)
	\end{array}.\right.
\end{equation}
In score-based setting, the objective function for the baselines is the predicted probability on the target label, i.e., $P(\mathbf{x},\mathbf{c}_{tar})$.
For HydraText, $f_1(S)$ is defined as follows:
\begin{equation}
	\textbf{Score-based}:\ f_1(S)=\left\{\begin{array}{ll}
		1, & \mathbf{c}_{tar} = \mathbf{c}_{adv}\\
		P(\mathbf{x},\mathbf{c}_{tar}), & \text{otherwise}
	\end{array}.\right.
\end{equation}
In decision-based setting, GADe will first initialize a $\mathbf{x}$ with $\mathbf{c}_{tar} = \mathbf{c}_{adv}$ by random word substitution.
If it does not find a solution that satisfies the requirement in the initialization phase, then the attack will fail.
For HydraText, $f_1(S)$ is defined as follows:
\begin{equation}
	\textbf{Decision-based}:\ f_1(S)=\left\{\begin{array}{ll}
		+\infty, & \mathbf{c}_{tar} = \mathbf{c}_{adv}\\
		|S|, & \text{otherwise}
	\end{array}.\right.
\end{equation}

Table~\ref{tab:targeted} presents the performance of HydraText and the baselines in crafting targeted attacks.
The results are similar to the ones in Table~\ref{tab:main_results}.
Compared to the baselines, HydraText could achieve quite competitive attack success rates and better attack imperceptibility.
One may observe that all these approaches performed worse in crafting targeted attacks than in crafting untargeted ones (see Table~\ref{tab:main_results}).
This is conceivable because a successful targeted attack is also a successful untargeted attack, while the opposite is unnecessarily true.

\section{Conclusion}
\label{sec:conclusion}
In this paper, we leveraged multi-objectivization to integrate attack imperceptibility into the problem of crafting adversarial textual attacks.
We proposed an algorithm, dubbed HydraText, to solve this problem.
Extensive experiments demonstrated that compared to existing textual attack approaches, HydraText could consistently craft more indistinguishable and natural AEs, without compromising on attack success rates.

The proposal of HydraText extends the realm of multi-objectivization to adversarial textual attack, where more objectives such as ﬂuency and naturality can be further incorporated.
Meanwhile, the strong performance of HydraText is suggestive of its promise in other domains such as image and speech, where attack imperceptibility is also an essential consideration.
Moreover, as discussed in Section~\ref{sec:results}, improving the query efficiency of HydraText is also worth studying.
Finally, it is also interesting to investigate how to automatically build an ensemble of textual attacks \cite{LiuT019,LiuP023} to reliably evaluate the adversarial robustness of NLP models.


\bibliographystyle{unsrtnat}
\bibliography{mybib}

\appendix

\begin{table*}[b]
	\color{black}
	\centering
	\caption{\color{black}The average computation time (in seconds) for 1000 queries to each target model on each dataset.}
	\scalebox{0.8}{
	  \begin{tabular}{cccc}
	  \toprule
			& BERT  & WordLSTM & WordCNN \\
	  \midrule
	  IMDB  & 8.98  & 1.66  & 0.33 \\
	  MR    & 11.02 & 2.26  & 1.13 \\
	  AG News & 12.41 & 1.56  & 0.73 \\
	  \midrule
			& BERT  & Infersent & ESIM \\
	  \midrule
	  SNLI  & 9.09  & 7.27  & 7.06 \\
	  MNLI  & 9.66  & 8.17  & 6.25 \\
	  \bottomrule
	  \end{tabular}}
	\label{tab:query_time}%
\end{table*}

{
\color{black}
\section{Computation Time}
For adversarial textual attack, model queries, i.e., fitness evaluations (FEs), dominate the computational costs.
Specifically, evaluating a solution (AE) means feeding it to the target model and then obtaining the model output.
Table~\ref{tab:query_time} lists the average computation time for 1000 queries to each target model on each dataset.
By combining it with the average query number consumed by the attack approaches in Table~\ref{tab:main_results}, one can obtain the total computation time consumed by each attack approach to attack a testing instance.
}

\section{Human Evaluation}
Figure~\ref{fig:human_eval} presents an example in the questionnaire survey of the human evaluation study.

\label{appendix:humaneval}
\begin{figure}[b]
\centering
\begin{subfigure}[b]{0.4\linewidth}
	\centering
	\scalebox{1.0}{\includegraphics[width=\linewidth]{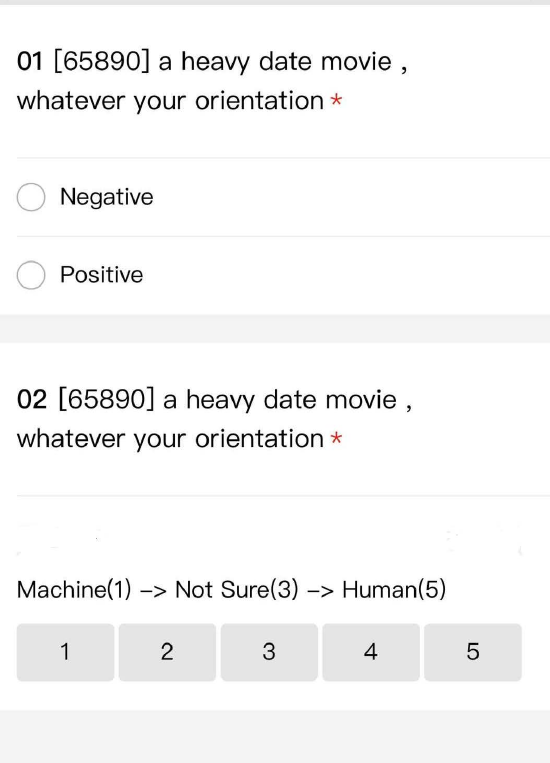}}
\end{subfigure}
\caption{An example in the questionnaire survey.}
\label{fig:human_eval}
\end{figure}







\end{document}